\documentclass[twoside]{article}

\usepackage[accepted]{aistats2026}

\usepackage[utf8]{inputenc} %
\usepackage[T1]{fontenc}    %
\usepackage{hyperref}       %
\usepackage{url}            %
\usepackage{booktabs}       %
\usepackage{amsfonts}       %
\usepackage{nicefrac}       %
\usepackage{microtype}      %
\usepackage{graphicx}
\usepackage{algorithm}
\usepackage{algorithmic}
\usepackage{amsmath} 
\usepackage{amssymb}
\usepackage{caption}
\usepackage{enumitem}
\usepackage{subcaption}
\usepackage{caption}
\usepackage{float}
\usepackage{amsthm}
\usepackage{cleveref}

\usepackage{chngcntr} %
\usepackage[dvipsnames]{xcolor}
\usepackage{tikz}
\usetikzlibrary{arrows.meta,positioning,fit,calc}

\hypersetup{
colorlinks,
linkcolor=NavyBlue
,citecolor=PineGreen
,filecolor=Mulberry
,urlcolor=NavyBlue
,menucolor=BrickRed
,runcolor=Mulberry
}

\newtheorem{observation}{Observation}
\newtheorem{hypothesis}{Hypothesis}

\newtheorem{proposition}{Proposition}
\newtheorem*{proposition*}{Proposition}
\newtheorem{theorem}{Theorem}

\crefname{hypothesis}{Hypothesis}{Hypotheses}
\Crefname{hypothesis}{Hypothesis}{Hypotheses}
\crefname{observation}{Observation}{Observations}
\Crefname{observation}{Observation}{Observations}
\Crefname{theorem}{Theorem}{Theorems}
\crefname{appendix}{appendix}{appendices}
\Crefname{appendix}{Appendix}{Appendices}

\usepackage[round]{natbib}

\begin{document}

\runningtitle{Standard Acquisition Is Sufficient for Asynchronous Bayesian Optimization}

\twocolumn[

\aistatstitle{Standard Acquisition Is Sufficient for\\Asynchronous Bayesian Optimization}

\aistatsauthor{ Ben Riegler \And James Odgers \And  Vincent Fortuin }

\aistatsaddress{ Technical University of Munich \\ Helmholtz AI \\ MCML \\ \texttt{ben.riegler@tum.de} \And  UTN \\ Helmholtz AI \\ MCML \And UTN \\ Helmholtz AI\\ MCML } ]

\begin{abstract}
Asynchronous Bayesian optimization is widely used for gradient-free optimization in domains with independent parallel experiments and varying evaluation times. Existing methods posit that standard acquisitions lead to redundant and repeated queries, proposing complex solutions to enforce diversity in queries. 
Challenging this fundamental premise, we show that methods, like the Upper Confidence Bound, can in fact achieve theoretical guarantees essentially equivalent to those of sequential Thompson sampling. 
A conceptual analysis of asynchronous Bayesian optimization reveals that existing works neglect intermediate posterior updates, which we find to be generally sufficient to avoid redundant queries. Further investigation shows that by penalizing busy locations, diversity-enforcing methods can over-explore in asynchronous settings, reducing their performance. Our extensive experiments demonstrate that simple standard acquisition functions match or outperform purpose-built asynchronous methods across synthetic and real-world tasks.
\end{abstract}

\section{INTRODUCTION}
\label{sec:intro}

A common problem in many scientific, machine learning, and engineering applications is optimizing functions, which may be evaluated point-wise via running an experiment, but for which there is no known analytical form and no gradient information. Such tasks are known as black-box optimization problems and are often approached using Bayesian optimization (BO). Applications include materials discovery \citep{packwood2017bayesian,cinquin2025actually}, chemical design \citep{griffiths2020constrained, wang2022bayesian}, as well as nuclear and accelerator physics \citep{roussel2024bayesian, ekstrom2019bayesian}. The main application in machine learning is hyperparameter tuning \citep{chen2018bayesian, snoek2012practical, wu2019hyperparameter, klein2017fast}.

\begin{figure*}[t]
    \centering

    \includegraphics{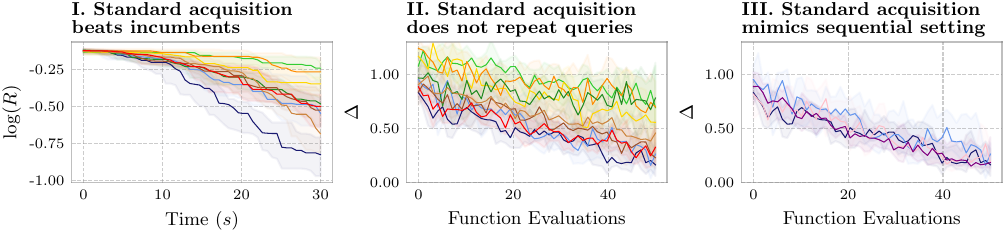}

    \includegraphics{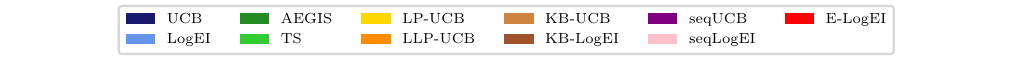}

  \caption{We find that (I.) standard acquisition outperforms or matches purpose-built methods for asynchronous Bayesian optimization. (II.) Comparing the distances of queries to the closest busy location, we see that standard acquisition exhibits the desirable transition from exploration to exploitation, but does not repeat queries. (III.) Standard acquisition functions query at similar distances as their optimally informed sequential counterparts, suggesting that nearby sampling is a desirable feature and should not be prohibited. See \Cref{sec:prelim} for a description of the methods and \Cref{sec:query_dists} for a formal definition of the distance $\Delta$. Results are shown on Ackley~$(d=10, q=8)$.}\label{fig:fig1}
  
\end{figure*}

In practice, it is often possible to run many experiments in parallel.
These might correspond to different devices in a wet lab or different GPUs on a compute cluster.
Due to varying function evaluation times, practitioners may opt for an asynchronous BO approach to minimize idle times \citep{zhang2020efficient,tran2022aphbo,egele2023asynchronous,koyama2022bo,frisby2021asynchronous}. In asynchronous BO, whenever a worker becomes available, a new input location must be selected without access to the outcomes of pending experiments. This challenge leads to the following, seemingly reasonable, hypothesis.
\begin{hypothesis}[Standard acquisition is wasteful in asynchronous BO]
\label{hyp:claim}
    Due to the unknown pending evaluations, it is necessary to explicitly enforce diversity in asynchronous BO queries. Standard acquisition functions' failure to do so will result in repeated or redundant queries, wasting evaluation resources, thus leading to poor BO results.
\end{hypothesis}
All existing works on asynchronous BO---explicitly or implicitly---build on \Cref{hyp:claim}, by proposing acquisition rules enforcing diversity in the asynchronous queries.
In this work, we perform the first critical examination of \Cref{hyp:claim}, both conceptually and empirically. We argue that the reasoning in \Cref{hyp:claim} is false, and by extension, proposed methods based on it may be unnecessary.

Instead, we propose to use existing standard acquisition rules, such as the Upper Confidence Bound (UCB) \citep{auer2002using,srinivas2010gaussian}  or Expected Improvement (EI) \citep{movckus1975bayesian}.
We motivate this theoretically by bounding the regret of the asynchronous standard UCB algorithm. Intuitively, we argue that 
\Cref{hyp:claim} fails to account for the update to the surrogate model with a new datum, which provides sufficient diversity in asynchronous updates. %

Empirically, we show excellent performance on an extensive suite of experiments, supporting our theoretical and conceptual insights regarding the inadequacy of \Cref{hyp:claim}. This is further investigated in an analysis of the distances of asynchronous queries to currently busy locations, revealing a similar exploration-exploitation trade-off as that in the optimally informed sequential BO.

Our contributions can be summarized as follows:
\begin{itemize}

    \item By bounding the Bayes simple regret, we show that the standard UCB in asynchronous BO achieves theoretical guarantees essentially equivalent to those of sequential Thompson sampling.

    \item We identify a conceptual flaw in the reasoning of \Cref{hyp:claim}, as it fails to account for information in the newly available datum before making the next asynchronous acquisition.

    \item We demonstrate that standard methods, such as the UCB or EI, in asynchronous BO match or outperform purpose-built ones (\Cref{fig:fig1}, I.), on synthetic and real-world optimization tasks, providing clear evidence against \Cref{hyp:claim}. 

    \item Contrary to \Cref{hyp:claim}, we show empirically that standard approaches do not query at and rarely close to currently busy locations (\Cref{fig:fig1}, II.), and present evidence that it may not be desirable to explicitly enforce diversity at all (\Cref{fig:fig1}, III.).

\end{itemize}

\section{PRELIMINARIES}
\label{sec:prelim}

\paragraph{Formal problem statement} In this work, we consider the global optimization of real-valued functions, $f: \mathcal{X} \mapsto \mathbb{R}$, on some compact domain $\mathcal{X} \subseteq \mathbb{R}^d$. It is assumed that $f(\cdot)$  may be queried at some point in the input space $x \in \mathcal{X}$, resulting in a time-delayed observation of the corresponding output $y \in \mathbb{R}$. The evaluation time varies throughout the input space. After $n$ completed evaluations, $\mathbf{y} = \{y_i\}_{i =1}^n$, at inputs $X = \{x_i\}_{i =1}^n$, we have data $D_n$. In the asynchronous setting with $q$ workers, the unknown function values at busy locations, $\mathcal{B} = \{x_j\}_{j =1}^{q-1}$, are denoted by $\mathbf{y}_b = \{y_j\}_{j =1}^{q-1}$, which we collect in unobserved data $D_b$. 

\begin{algorithm}[t]
    \caption{ \ \ Asynchronous BO with $q$ workers}
    \label{alg:asyncBO}
    \begin{algorithmic}[1]
        \REQUIRE Oracle $f(\cdot)$, acquisition function $\alpha(\cdot)$, globally stored data $D_0 = \{(x_i, y_i)\}_{i=1}^{n_0}$, time budget $T$
        \STATE \textbf{initialize:} $\mathcal{B} \leftarrow $ $\{x_{j}\}_{j=1}^q$ quasi-random
        \STATE \textbf{spawn:} $q$ worker processes
        \STATE start timer
        \FOR{$x_j \in \mathcal{B}$ }
            \STATE start worker $j$ with query $f(x_j)$
        \ENDFOR
        \WHILE{elapsed time $ < T$}
            \IF{ $j$ completes query $(x_{n+1}, y_{n+1})$} 
                \STATE $D_{n+1} \leftarrow D_n \cup \{(x_{n+1}, y_{n+1})\}$
                \STATE $\mathcal{B} \leftarrow \mathcal{B} \setminus \{x_{n+1}\}$
                \STATE $\mathcal{GP}_j \leftarrow$ fit-surrogate($\mathcal{GP}_j$, $D_{n+1}$)
                \STATE $x' \leftarrow \arg\max \limits_{x \in \mathcal{X}} \ \alpha(x \mid D_{n+1})$ 
                \STATE $\mathcal{B} \leftarrow \mathcal{B} \cup \{x'\}$
                \STATE start worker $j$ with query $f(x')$
            \ENDIF
        \ENDWHILE
        \RETURN $i^* = \arg\max_{i} y_i$ and $(x_{i^*}, y_{i^*})$
    \end{algorithmic}
\end{algorithm}

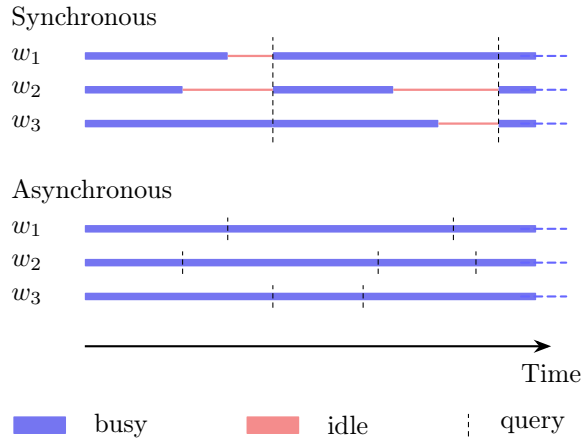
\begin{figure}[h]
  \centering

\begin{tikzpicture}[
    >={Stealth[length=3mm]},
    lane/.style={draw=none, fill=none},
    tracklabel/.style={anchor=west, align=left},
    segbusy/.style={fill=blue!90!black!60!white!90},
    segidle/.style={fill=red!90!black!50!white!90},
    query/.style={draw=black, line width=0.4pt, dash pattern=on 2pt off 1.5pt},
    axis/.style={-Stealth, thick}
]
\def\L{6}       %
\def\laneH{1.9}   %
\def\rowH{0.45}   %
\def\barH{0.10}   %
\def\gX{1.5}      %
\def\gY{0.6}      %

\def\barHbusy{0.10}   %
\def\barHidle{0.03}

\def\labelX{\gX-1.0}   %
\def\headsep{0.5}     %

\def\sepH{0.30}   %

\def\busyLineW{0.8pt}                 %
\def\busyDash{on 3pt off 2pt}         %

\tikzset{
  busyline/.style={
    draw=blue!60!white,
    line width=0.8pt,
    dash pattern=on 3pt off 2pt,
    line cap=round
  }
}

\newcommand{\busyline}[3]{%
  \draw[busyline] (#2, #1) -- ++(#3, 0);
}

\node[lane, minimum width=\L cm, minimum height=\laneH cm, anchor=west] (sync)  at (\gX, \gY+\laneH+0.4) {};
\node[lane, minimum width=\L cm, minimum height=\laneH cm, anchor=west] (async) at (\gX, \gY) {};

\newcommand{\barLLH}[5]{%
  \path[#5, draw=none, rounded corners=0.3pt]
    (#2, {#1-0.5*#4}) rectangle ++(#3, #4);
}
\newcommand{\barBusy}[3]{\barLLH{#1}{#2}{#3}{\barHbusy}{segbusy}}
\newcommand{\barIdle}[3]{\barLLH{#1}{#2}{#3}{\barHidle}{segidle}}

\newcommand{\barLL}[4]{%
  \path[#4, draw=none, rounded corners=0.3pt]
    (#2, {#1-0.5*\barH}) rectangle ++(#3, \barH);
}
\newcommand{\barLR}[4]{%
  \path[#4, draw=none, rounded corners=0.3pt]
    (#2, {#1-0.5*\barH}) rectangle (#3, {#1+0.5*\barH});
}

\newcommand{\sepmark}[3][\sepH]{%
  \draw[query] (#2, {#3-0.5*#1}) -- (#2, {#3+0.5*#1});
}

\newcommand{\seprows}[3]{%
  \draw[query] (#1, {#3-0.5*\barH}) -- (#1, {#2+0.5*\barH});
}

\def\sa{ \gY+\laneH+0.4 + \laneH - 0.75*\rowH}
\def\sb{ \gY+\laneH+0.4 + \laneH - 1.75*\rowH}
\def\sc{ \gY+\laneH+0.4 + \laneH - 2.75*\rowH}
\def\ya{ \gY + \laneH - 0.75*\rowH}
\def\yb{ \gY + \laneH - 1.75*\rowH}
\def\yc{ \gY + \laneH - 2.75*\rowH}

\node[tracklabel] at (\labelX, \sa+\headsep) {Synchronous};
\node[tracklabel] at (\labelX, \ya+\headsep) {Asynchronous};

\node[tracklabel] at (\labelX, \sa) {$w_1$};
\node[tracklabel] at (\labelX, \sb) {$w_2$};
\node[tracklabel] at (\labelX, \sc) {$w_3$};
\node[tracklabel] at (\labelX, \ya) {$w_1$};
\node[tracklabel] at (\labelX, \yb) {$w_2$};
\node[tracklabel] at (\labelX, \yc) {$w_3$};

\barLL{\sa}{\gX+0.1}{1.9}{segbusy}
\barIdle{\sa}{\gX+1.6+0.4}{0.6}
\barLL{\sa}{\gX+2.6}{3.5}{segbusy}
\busyline{\sa}{\gX+5.9}{0.6}

\barLL{\sb}{\gX+0.1}{1.3}{segbusy}
\barIdle{\sb}{\gX+1.4}{1.2}
\barLL{\sb}{\gX+2.6}{1.6}{segbusy}
\barIdle{\sb}{\gX+4.2}{1.4}
\barLL{\sb}{\gX+5.6}{0.5}{segbusy}
\busyline{\sb}{\gX+5.9}{0.6}

\barLL{\sc}{\gX+0.1}{2.5}{segbusy}
\barLL{\sc}{\gX+2.6}{2.2}{segbusy}
\barIdle{\sc}{\gX+4.8}{0.8}
\barLL{\sc}{\gX+5.6}{0.5}{segbusy}
\busyline{\sc}{\gX+5.9}{0.6}

\seprows{\gX+2.6}{\sa+0.2}{\sc-0.2}
\seprows{\gX+5.6}{\sa+0.2}{\sc-0.2}

\barLL{\ya}{\gX+0.1}{6}{segbusy}
\sepmark{\gX+2.0}{\ya}
\sepmark{\gX+5.0}{\ya}
\busyline{\ya}{\gX+5.9}{0.6}

\barLL{\yb}{\gX+0.1}{6}{segbusy}  %
\sepmark{\gX+1.4}{\yb}
\sepmark{\gX+4}{\yb}
\sepmark{\gX+5.3}{\yb}
\busyline{\yb}{\gX+5.9}{0.6}

\barLL{\yc}{\gX+0.1}{6}{segbusy}  %
\sepmark{\gX+2.6}{\yc}
\sepmark{\gX+3.8}{\yc}
\busyline{\yc}{\gX+5.9}{0.6}

\draw[axis] (\gX+0.1, 0.6) -- ++(\L+0.2,0) node[below, yshift=-3pt] {Time};

\coordinate (legendCenter) at ({\gX + 0.5*(\L+0.2)}, -0.45);

\begin{scope}[shift={(legendCenter)}]
  \node[segbusy, draw=none, rounded corners=0.3pt,
        minimum width=0.7cm, minimum height=0.18cm] at (-3.6, 0) {};
  \node[anchor=west] at (-3.0, 0) {busy};

  \node[segidle, draw=none, rounded corners=0.3pt,
        minimum width=0.7cm, minimum height=0.18cm] at (-0.5, 0) {};
  \node[anchor=west] at (0.1, 0) {idle};

  \draw[query] (2.1, -0.15) -- (2.1, 0.15); %
  \node[anchor=west] at (2.4, 0) {query};
\end{scope}

\end{tikzpicture}

  \caption{Synchronous vs.\ asynchronous BO with $q=3$ workers. The asynchronous BO can run more experiments in the same overall time.}
  \label{fig:parBO}
\end{figure}

\paragraph{Bayesian optimization} Designed for expensive-to-evaluate black-box objective functions, Bayesian optimization is a sample-efficient global optimization framework \citep{jones1998efficient}.  In order to optimize the objective, $f(\cdot)$, with mere point-wise evaluation, Bayesian optimization represents the uncertainty in the objective via a probabilistic surrogate model, built from $D_n$. The surrogate then informs an acquisition function, $\alpha: \mathcal{X} \mapsto \mathbb{R}$, proposing the next query location $x' \in \mathcal{X}$. 

See \Cref{fig:parBO} for an illustration of synchronous vs.\ asynchronous BO and \Cref{alg:asyncBO} for pseudo code of asynchronous BO with $q$ workers. Importantly, in this work we consider a fully distributed set up, where each worker performs all steps of the BO algorithm locally, with access to the globally shared data set, $D_n$. Thus, the only source of idle workers is reading from and writing to $D_n$. These operations are protected by a lock, and the resulting idle times are negligible in practice.

\paragraph{Gaussian Process surrogate}  A zero-mean Gaussian Process (GP) prior, $\mathcal{GP}(0, k_\phi(\cdot, \cdot) )$, is placed on $f(\cdot)$ and defined through a positive definite covariance function $k_\phi: \mathcal{X} \times\mathcal{X} \mapsto \mathbb{R}$, with hyperparameters $\phi$, including lengthscales $\ell \in \mathbb{R}^d$. The GP is a standard choice, as it allows for the incorporation of prior knowledge on smoothness, periodicity, and trend of $f(\cdot)$, as well as an analytical posterior \citep{williams2006gaussian}. 

For a Gaussian observation model with noise $\eta^2$ and data $ D_n = (X, \mathbf{y})$, the posterior predictive distribution at any given input location, $x$, is Gaussian with,
\begin{align}
    p(f(x) \mid D_n )= \mathcal{N}(f(x)\mid \mu(x \mid D_n), \sigma^2(x\mid D_n)).
\end{align}
The posterior mean 
and variance 
are analytically tractable as, 
\begin{align}
    \mu(x \mid D_n) &= k_\phi(x, X)M_{XX}\mathbf{y} \label{equ:post_mean} \\  
    \sigma^2(x \mid D_n) &= k_\phi(x,x) 
    -k_\phi(x, X)M_{XX}k_\phi(X, x) \label{equ:post_var}
\end{align}
with $M_{XX} \equiv (K_{XX} + \eta^2\mathbf{I}_n)^{-1}$ and $K_{XX}$ the kernel matrix of input locations $X$, i.e., $[K_{XX}]_{ij} = k_{\phi}(x_i, x_j) \ \forall \ i,j \in [n]$. 

\paragraph{Acquisition functions in asynchronous BO}
\label{sec:acqf}
In sequential and asynchronous BO alike, a new input location $x' \in \mathcal{X}$ is chosen, as soon as an evaluation resource (e.g., a GPU) becomes available.  The choice, $x'$, is made by maximizing the acquisition function. 

Standard acquisition functions considered in this work are the Upper Confidence Bound (UCB) \citep{auer2002using, srinivas2010gaussian} and the Expected Improvement (EI) \citep{movckus1975bayesian, jones1998efficient}. For numerical reasons \citep{ament2023unexpected} we work with the natural logarithm of the EI (LogEI). Due to the insistence on \cref{hyp:claim}, to the best of our knowledge, these standard heuristics have not been considered in previous literature, not even as simple baselines. Instead, numerous works have devised elaborate solutions, which can be grouped as follows. 

In asynchronous BO, \textit{Monte Carlo} sampling methods, such as the expected LogEI (E-LogEI), estimate the expected acquisition function given $q-1$ busy locations, $\mathcal{B} \in \mathcal{X}^{q-1}$, under the surrogate \citep{ginsbourger2011dealing, janusevskis2012expected, snoek2012practical}. \textit{Hallucination} heuristics, such as the Kriging Believer (KB), replace the $q-1$ unknown function values with some value, e.g., the posterior mean \citep{ginsbourger2010kriging, ginsbourger2007multi}. \textit{Randomness}-based heuristics, e.g., Thompson sampling (TS) or AEGIS, introduce a sampling step in the acquisition, aiming to ensure query diversity \citep{kandasamy2018parallelised, de2021greed}. \textit{Penalization}-based methods (LP and LLP) down-weight the acquisition function at and around busy locations \citep{alvi2019asynchronous}. A more detailed discussion of these methods can be found in \Cref{sec:app_acqf}.

\begin{figure*}[t]
    \centering

    \includegraphics{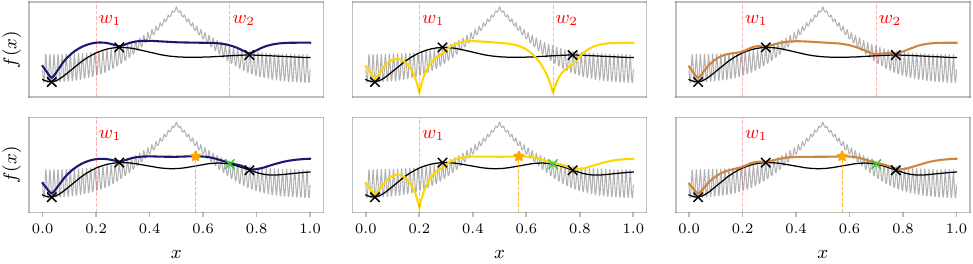}

    \includegraphics{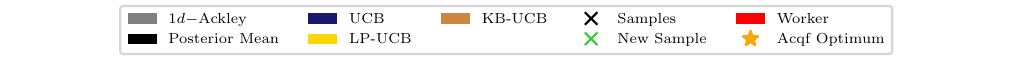}

    \caption{One full asynchronous BO step for the UCB (left), LP-UCB (middle), and the KB-UCB (right). Row 1 corresponds to line 6 in \Cref{alg:asyncBO}, with $n_0 = 3$ initial samples and $q=2$ initialized workers, $w_1$ and $w_2$. In row 2, $w_2$ finishes first, the GP-surrogate is updated with the new sample (green), and the acquisition function is optimized (orange). Notably, all methods query almost the same location, even the standard UCB, which does not take into account the busy locations. See \Cref{fig:app_iter} for subsequent iterations showing a similar trend.}

    \label{fig:iter}
\end{figure*}

\begin{figure*}[t]
    \centering

    \includegraphics{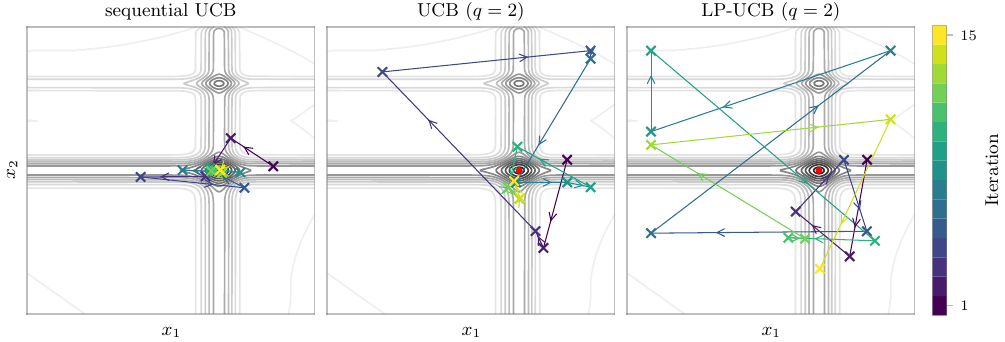}

    \includegraphics{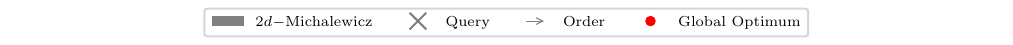}

    \caption{Sample trajectories for the sequential UCB, the asynchronous (standard) UCB, and the purpose-built LP-UCB, with the same initial data.
    Like sequential UCB, the asynchronous UCB discovers the global optimum (red) within the budget of $15$ iterations, despite performing the occasional close query. The \textit{penalization}-based method LP-UCB over-explores the search space and does not find the optimum within the budget. See \Cref{fig:app_2d} for more runs and different methods.}
    \label{fig:2d}
\end{figure*}

\section{THEORETICAL ANALYSIS}  
\label{sec:theoretical}

The prior works presented in \Cref{sec:prelim}, concerned with resolving \Cref{hyp:claim}, do not provide any theoretical or empirical evidence for repeated or redundant queries, despite using them as a motivation for their methods. In this section, we present theoretical arguments for why standard acquisition can work well in asynchronous BO and why certain purpose-built methods might not. 

\subsection{Bayes simple regret bound for asynchronous UCB}

\label{sec:bound}

While not considering busy locations, $\mathcal{B}$, in the acquisition may seem counter-intuitive, this is a commonality of our approach with work on asynchronous TS by \citet{kandasamy2018parallelised}. Perhaps surprisingly, the authors show that asynchronous TS achieves essentially equivalent theoretical performance for a given number of observations, $n$, as its optimally informed sequential counterpart. This is supported by showing that the respective Bayes simple regrets (BSR) after $n$ completed evaluations,
\begin{align}
    \text{BSR}(n) = \mathbb{E}[f(x^*)- \max_{i \in [n]} f(x_i)],
\end{align}
admit upper bounds of asymptotically the same order, $\mathcal{O}(\cdot)$, in $n$. The expectation is w.r.t.~the prior on $f$, the observation noises $\{{\epsilon_i}\}_{i=1}^n$, as well as any randomness in the algorithm, e.g., GP posterior samples in TS. 

For the asynchronous UCB, we present the following bound on the BSR, which is equivalent to the bound on asynchronous TS in \citet{kandasamy2018parallelised}.

\begin{theorem}[Informal. Bound on BSR$(n)$ for asynchronous UCB]
\label{thm:bound}
    Let $f \sim \mathcal{GP}(0, k_{\phi}(\cdot, \cdot))$. Then, for the UCB acquisition function used asynchronously (see \Cref{alg:asyncBO}), the Bayes simple regret after $n$ queries can be bounded as
    \begin{align}
        \textnormal{BSR}(n) \lesssim  \sqrt{\frac{\xi_q \log(n) \Psi_n}{n}}.
    \end{align} 
\end{theorem}
\begin{proof}
    A formal statement and proof are given in \Cref{sec:app_bound}.
\end{proof}

The bound is given in terms of the maximum information gain (MIG), $\Psi_n$, after $n$ evaluations, and the quantity $\xi_q$. Under an RBF kernel, the MIG grows as $\Psi_n \propto \log(n)^{d+1}$ \citep{srinivas2010gaussian}, i.e., \mbox{sub-linearly}. Intuitively, $\xi_q$ represents the cost of operating asynchronously, quantifying the missing information in the $q-1$ busy locations, and is bounded independently of $n$. See \Cref{sec:app_bound} for formal definitions of both $\Psi_n$ and $\xi_q$.

By the same reasoning as in \citet{kandasamy2018parallelised}, this provides strong theoretical motivation for why the standard UCB should work well in the asynchronous setting. Namely, its BSR$(n)$ admits an upper bound of the same asymptotic order as that of the optimally informed sequential TS. In particular, this shows that good performance can be theoretically guaranteed, even when neither accounting for busy locations nor injecting randomness. Further intricacies, including the bounds presented by \citet{kandasamy2018parallelised} are discussed in \Cref{sec:app_bound}.

\subsection{The expected UCB is the Kriging Believer}
\label{sec:margin}

Under the GP surrogate, the unknown function values,  $\mathbf{y}_b = \{y_j\}_{j=1}^{q-1}$, at the busy locations, $\mathcal{B} = \{x_j\}_{j=1}^{q-1}$,  render the acquisition function, $\alpha(x \mid D_n,  (\mathcal{B}, \mathbf{y}_b))$, a random function. As \citet{ginsbourger2011dealing} point out, it is not clear how to deal with this randomness, since the optimization of a random function is not a well-defined task. The proposed solution is to optimize the expected acquisition function, in particular the EI. Due to the intractability of this expectation, this is estimated from \textit{Monte Carlo} samples (see \Cref{sec:app_incumb}).

 Using a different approach, \textit{penalization}-based methods (LP and LLP) by \citet{alvi2019asynchronous} attempt to approximate the expected UCB as
\begin{align}
    \alpha_{LP}(x\mid \mathcal{B}) &\approx \mathbb{E}[\alpha_{UCB}(x \mid D_n,  D_b)\mid D_n, \mathcal{B}] \label{equ:expect}\\
    &= \int \alpha_{UCB}(x \mid D_n,  D_b) \ p(\mathbf{y}_b\mid D_n) d\mathbf{y}_b. \label{equ:integ}
\end{align}

Unlike for the EI, we show that for the UCB, this integral is in fact analytically tractable and leads to a well-known heuristic.

\begin{proposition}[Expected UCB is the Kriging Believer]
\label{prop:ucb_KB}
Consider the random Upper Confidence Bound, $\alpha_{UCB}(x|D_n, (\mathcal{B}, \mathbf{y}_b))$, of the GP surrogate posterior, with unobserved function values, $\mathbf{y}_b$, at known busy locations, $\mathcal{B}$. Then it holds that
\begin{equation}
\label{equ:prop}
\begin{split}
\mathbb{E}[\alpha_{UCB}(x \mid D_n, (\mathcal{B}, \mathbf{y}_b)) \mid D_n, \mathcal{B}] \\
= \alpha_{UCB}(x \mid D_n, (\mathcal{B}, \mu_b))
\end{split}
\end{equation}
with $\mu_b$ the GP posterior mean at $\mathcal{B}$.
\end{proposition}
\begin{proof}
    The proof is given in \Cref{sec:app_proof}.
\end{proof}
We use the fact that $\mu(\cdot \mid D_n, D_b)$ is linear in $\mathbf{y}_b$ and that $\sigma(\cdot\mid D_n,  D_b)$ only depends on the known input locations, $X$ and $\mathcal{B}$, but not on the unknown function values.

From the above result, it can be seen that marginalizing over the values at busy locations gives the UCB we would get by simply assuming the posterior mean, $\mu_b$, at $\mathcal{B}$. This is a \textit{hallucination}-based heuristic well known as the Kriging Believer \citep{ginsbourger2010kriging}. On the other hand, \Cref{prop:ucb_KB} then gives an additional insight into the Kriging Believer. We have now shown that, in the case of the UCB, the Kriging Believer is the expected acquisition function.

We have thus shown that the UCB \textit{penalization}-based methods, as well as the Kriging Believer, are constructed from the expected acquisition function, not taking into account higher-order moments of the random acquisition function $\alpha(x \mid D_n,  (\mathcal{B}, \mathbf{y}_b))$. This provides a theoretical underpinning for the poor performance of these methods \citep{alvi2019asynchronous,de2021asynchronous}, which we further demonstrate in \Cref{sec:emp_tasks}.

\begin{figure}[t]
    \centering

    \includegraphics{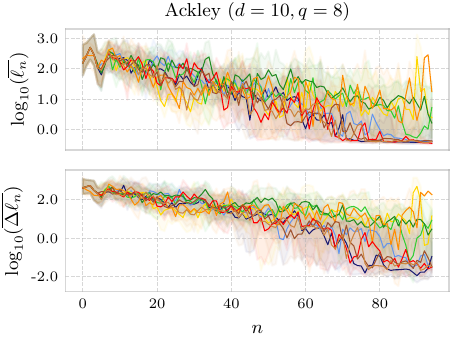}

    \includegraphics{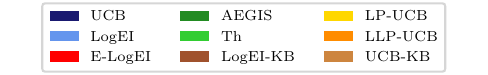}

    \caption{
    \textit{Top:} The average ARD-RBF kernel lengthscale decreases by about two orders of magnitude from initialization to a plateau for standard methods. Non-standard methods do not all reach this plateau. \textit{Bottom:} The mean absolute difference in subsequent lengthscales drops significantly from initialization to convergence for standard acquisition functions. Several non-standard methods query locations leading to non-converging kernel lengthscales. See \Cref{sec:hyp} for formal definitions of $\overline{\ell_n}$ and $\overline{\Delta\ell_n}$.} 
    \label{fig:hpy}
\end{figure}

\section{CONCEPTUAL ARGUMENTS}
\label{sec:concept}
After presenting some theoretical motivation for why standard acquisition can work well and other heuristics may not, we turn to a more conceptual investigation of \Cref{hyp:claim}. 

\subsection{Each new query is selected under a different---and more informed---GP posterior}
\label{sec:observation}

As described in line 1 of  \Cref{alg:asyncBO}, the $q$ workers are initialized at quasi-random locations (e.g., via Halton sequence \citep{halton1964algorithm}) in the search space, $\mathcal{X}$, guaranteeing initial diversity. Once a worker finishes its function evaluation, $y_{n+1}$, at $x_{n+1}$, the surrogate is updated and a new acquisition is made. This new acquisition is informed by the newly arrived datum $(x_{n+1}, y_{n+1})$, as well as all data collected by all workers prior, $D_{n+1}~=~D_n \cup (x_{n+1}, y_{n+1})$, and selected by
\begin{align}
    x' = \arg\max\limits_{x \in \mathcal{X}} \ \alpha(x\mid D_{n+1}).
\end{align}

\begin{observation}[\Cref{hyp:claim} neglects GP surrogate update]
\label{obs:update}
    \Cref{hyp:claim} states that the new input location, $x'$, will either be redundant or even repeated. But this fails to account for the update to the surrogate likelihood via the new datum, $(x_{n+1}, y_{n+1})$, as well as to the surrogate kernel hyperparameters, $\phi$. In fact, existing works presented in \Cref{sec:app_incumb} never discuss the intermediate update of the surrogate at all.
\end{observation}

\paragraph{Building intuition}
The stated goal of introducing purpose-built acquisition functions for asynchronous BO was to prevent redundant queries.
\Cref{obs:update} suggests that the problem of querying the same, or very similar, locations in asynchronous BO does not trivially arise from using standard acquisition functions.
Showing one asynchronous BO step for the UCB, the locally \textit{penalized} UCB (LP-UCB) and the \textit{hallucinating} Kriging Beliver UCB, \Cref{fig:iter} illustrates that the intermediate update can, at least in principle, be effective in guiding the subsequent acquisition. Without a need to explicitly enforce diversity, the new datum shifts the optimum of the UCB into a promising region, far from previous and currently busy queries. Importantly, the purpose-built methods taking into account the busy locations in rather elaborate ways end up querying the same point as the standard UCB. This effect persists in the following iterations, as can be seen in \Cref{fig:app_iter}.

Moreover, from \Cref{obs:update}, methods explicitly enforcing diversity in queries could be expected to suffer from an inefficient over-exploration of the search space. This is because, by preventing exploitation, the algorithm is forced to query less promising areas, which should have already been eliminated from consideration. We illustrate this in \Cref{fig:2d}, which shows the sample trajectories for three variants of the UCB heuristic. The optimally informed sequential UCB recovers the optimum within few iterations. An over-exploration of the search space by the \textit{penalization}-based method LP-UCB \citep{alvi2019asynchronous} prevents it from reaching the global optimum within the iteration budget. While the asynchronous UCB occasionally queries close to previous locations, against which there is no theoretical guarantee, it offers a better exploration-exploitation trade-off and thus reaches close to the global optimum.

\label{sec:synth_exper}
\begin{figure*}[t]
    \centering

    \includegraphics{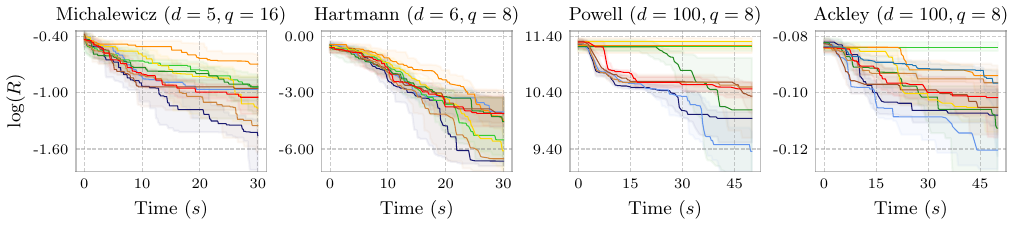}
    \includegraphics{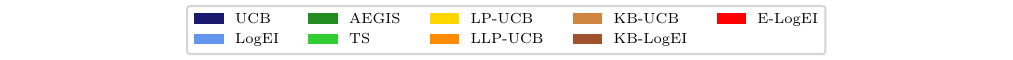}

    \caption{\textbf{Standard acquisition dominates on synthetic tasks.} This observation is robust to the dimensionality and the number of workers, whereas the performance of the purpose-built methods is always worse and more sensitive to task parameters (see also \Cref{fig:app_reg,fig:app_more_t}).}
    \label{fig:reg_synth}
\end{figure*}

\begin{figure*}[t]
    \centering

    \includegraphics{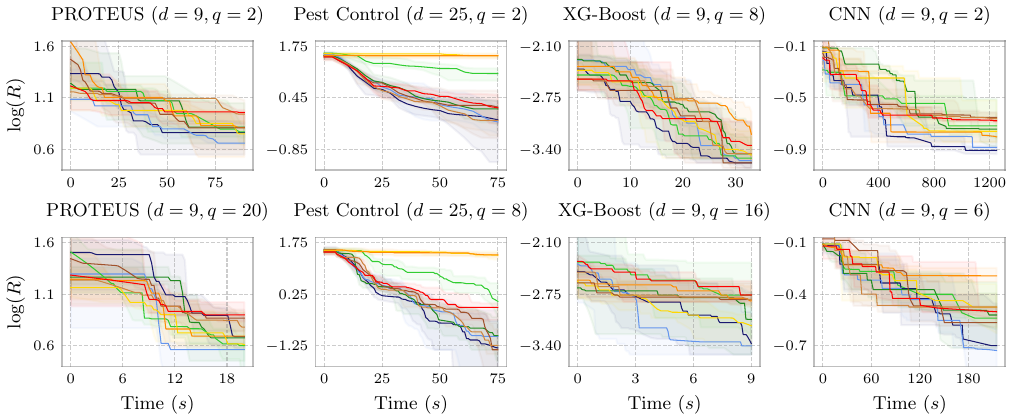}
    \includegraphics{figs/legend2.pdf}

    \caption{\textbf{Standard acquisition performs well on real-world and hyperparameter tuning tasks and benefits more from an increase in the number of workers than purpose built methods.} Both rows show the same optimization tasks, but with different numbers of workers, $q$. From the axes it can be seen how an increase in evaluation resources speeds up the optimization. Importantly, standard acquisition often benefits disproportionately from more workers, compared to existing methods.}
    \label{fig:reg_real}
\end{figure*}

\subsection{The kernel and likelihood updates naturally facilitate a good exploration-exploitation trade-off}
\label{sec:hyp}

Good results in BO require good exploration-exploitation trade-offs \cite{srinivas2010gaussian}. The concern of \cref{hyp:claim} is that a new observation will not produce a sufficiently large update to encourage exploration. We now disentangle the respective contribution to the exploration-exploitation trade-off of the two distinct updates presented in \Cref{obs:update}. To this end, we define
\begin{align}
    \overline{\ell_n} \equiv \frac{1}{d} \sum_{j = 1}^d \ell_{n,j} 
\end{align}   
and
\begin{align}
    \overline{\Delta \ell_n} \equiv \frac{1}{d} \sum_{j = 1}^d |\ell_{n,j} - \ell_{n-1, j}|,
\end{align}
where $\ell_{n, j} > 0$ is the lengthscale corresponding to dimension $j \in [d]$ of an ARD kernel, $k_{\phi}(\cdot, \cdot)$, as introduced in \Cref{sec:prelim}, at iteration $n$ of \Cref{alg:asyncBO}. While $\overline{\ell_n}$ is a measure for the covariance between function values under the GP-surrogate averaged over all $d$ axes, $\overline{\Delta \ell_n}$ reflects the change in this covariance compared to the previous iteration.

\paragraph{Exploration phase and kernel update} The updated kernel lengthscales, $\ell_{n+1} \in \mathbb{R}^d$, in particular, will affect the shape of the surrogate posterior and ultimately the acquisition. The lengthscales are updated together with the other kernel hyperparameters from $\phi_n$ to $\phi_{n+1}$ to fit the updated data, $D_{n+1}$ (line 11 in \Cref{alg:asyncBO}). This is done by optimizing the unnormalized posterior, $p(\phi| \mathbf{y}_{n+1})$, as
\begin{align}
    \phi_{n+1} &= \arg \max\limits_{\phi} \ \log p(\mathbf{y}_{n+1}\mid \phi) + \log p(\phi),
\end{align}
in a gradient-based manner. 

In \Cref{fig:hpy}, we find empirically that the magnitude of this update is largest in the early stages of the optimization and essentially vanishes towards the end. Intuitively, in the exploration phase with little data, the new datum will carry significant information on the lengthscales. Being a global property of the surrogate, the lengthscales can shift the acquisition function optimum, even if the new datum is in an unpromising region, thereby facilitating exploration of the search space.

\paragraph{Exploitation phase and likelihood update} For most common kernels, the nature of the likelihood update is local, forcing the GP posterior mean, $\mu(\cdot|D_{n+1})$, to pass close by the new datum. That is, $y_{n+1} \approx \mu(x_{n+1}|D_{n+1})$, depending on the observation noise level, $\eta^2$. Also, the uncertainty, expressed via the posterior variance $\sigma^2(\cdot | D_{n+1})$, collapses to at most the observation noise, $\eta^2$, at the input $x_{n+1}$, reflecting the gain in information at and around this location. Note that this collapse occurs even if $y_{n+1}$ already lies on the posterior mean, i.e., $y_{n+1} = \mu(x_{n+1}\mid D_n)$. This local update changes the acquisition function around $(x_{n+1}, y_{n+1})$. 

In the exploitation phase, kernel lengthscales have effectively converged for standard acquisition methods (\Cref{fig:hpy}). Thus, we can no longer rely on the kernel update later in the optimization, but we argue that this may not even be necessary. In the exploitation phase, queries \textit{should} lie close together, and the acquisition function optimum only needs to be shifted locally in each update. As argued above, this local shift is exactly what the likelihood update can provide.

\section{EXPERIMENTAL EVALUATION}
\label{sec:empirical}

We support our theoretical and conceptual insights from \Cref{sec:theoretical,sec:concept} with experiments on a range of synthetic, real-world and hyperparameter tuning tasks. After demonstrating excellent performance of standard acquisition functions (\Cref{sec:emp_tasks}), we investigate the exploration-exploitation trade-off made by both standard and purpose-built acquisition functions (\Cref{sec:query_dists}).

\subsection{Optimization tasks}
\label{sec:emp_tasks}

Following common practice, we report the logarithm of the simple regret, $R$, $\log(R) = \log \left| f^* - y_n^* \right|$, where $f^*$ is the known function optimum. We plot the median of the $\log(R)$ together with the inter-quartile range. For details on the implementation\footnote{We make our code available at \url{https://github.com/fortuinlab/SAIS-BO}} and tasks, please refer to \Cref{sec:app_impl,sec:app_exper}.

\paragraph{Synthetic test functions}

The acquisition rules are compared on synthetic test functions across a range of dimensions and numbers of workers. All experiments are repeated $20$ times with different initializations. In order to simulate the asynchronous setting, the evaluation time is sampled from a half-normal distribution with scale parameter $\theta = \sqrt{\nicefrac{\pi}{2}}$ \citep{alvi2019asynchronous,de2021asynchronous}. We choose functions used in works outlined in \Cref{sec:app_incumb}, to enable comparison of results.

In \Cref{fig:reg_synth}, it can be seen that no purpose-built method outperforms both the standard UCB and LogEI. While the KB-UCB does almost as well as the UCB on the $5d$–Michalewicz and $6d$–Hartmann functions, the \mbox{LogEI} is superior on the very high–dimensional Powell and Ackley functions. The \textit{penalization}-based methods, as well as TS, are unable to make meaningful progress in this vast search space.  A large number of additional results on more test functions and different configurations $(d,q)$, and with larger time budget can be found in \Cref{sec:app_results}.

\paragraph{Real-world tasks} The PROTEUS and Pest Control tasks showcase the effectiveness on problems beyond standard synthetic functions. In PROTEUS \citep{Lichtenberg_2021_JGRP}, a physics-based simulator for the evolution of planets, the objective is to retrieve the initial condition given a simulated planet, while Pest Control requires making a decision on herbicide at five different stations. See \Cref{sec:real_tasks} for a more detailed description.

Importantly, the PROTEUS experiment demonstrates how our findings continue to hold true, even in the very high worker setting of $q=20$ (see \Cref{fig:reg_real}). This is difficult to reconcile with \cref{hyp:claim}, where one might expect query redundancy to be exacerbated by a large number of workers. But, \Cref{fig:reg_real} reveals that this is not the case for either of these real-world experiments. 

\paragraph{Hyperparameter tuning tasks}
\label{sec:hpy_tasks}

To demonstrate the robustness of our findings, experiments are performed on two hyperparameter tuning tasks. Due to the large computational cost of these experiments (which is precisely what makes them relevant in this work on BO), they are repeated with $9$ different initializations. The tasks are briefly outlined below, and the results are shown in \Cref{fig:reg_real}. In \Cref{sec:real_tasks} we discuss the hyperparameter tuning tasks in more detail.

For XG-Boost, we consider nine hyperparameters, which we tune to optimize the 5-fold cross-validation accuracy. This score is computed from the classification performance on the UCI Breast Cancer data \citep{wolberg1993breast}. We also learn the hyperparameters of a 6-layer CNN pipeline to optimize validation accuracy, after 20 epochs of training on CIFAR10 \citep{krizhevsky2009learning}.

 Contrary to \Cref{hyp:claim}, we see in \Cref{fig:reg_real} that standard methods are able to effectively utilize parallel workers to accelerate the optimization, without the need for explicit diversity. The LogEI does extremely well on both tasks, while the UCB effectively matches this performance on the CNN task. A key insight here is that standard acquisition functions benefit disproportionately from an increase in the number of workers, $q$, compared to existing methods.

\subsection{Analysis of distances of queries from busy locations}
\label{sec:query_dists}

\begin{figure}[t]
    \centering

    \includegraphics{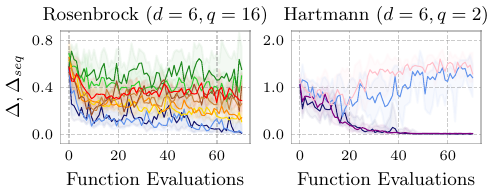}

    \includegraphics{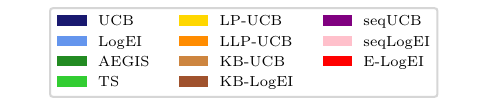}

    \caption{\textbf{Standard acquisition does not systematically repeat queries and performs an exploration-exploitation trade-off similar to that of its optimally informed sequential counterpart.} It can be seen that standard acquisition methods do not repeat queries, and that standard acquisitions query at similar distances to previous queries in the asynchronous setting as in the sequential setting.}
    \label{fig:dist}
\end{figure}

We now turn to an empirical investigation of our conceptual analysis in \Cref{sec:observation}. \Cref{hyp:claim} claims that in the absence of regularization, query distances to busy locations will be small or even zero. To quantify this behavior, we consider the distances of asynchronous queries to currently busy locations. As a gold standard, we also introduce the distances of the optimally informed sequential BO queries to the $q-1$ most recently sampled points. This allows for comparison of the exploration-exploitation trade-off made by all methods considered. Formally, we define the distance, $\Delta$, of the asynchronous query, $x'$, from busy locations, $\mathcal{B}$, as
\begin{align}
    \Delta \equiv \min\limits_{x_j \in \mathcal{B}} \lVert x' - x_j \rVert_2,
\end{align}
and analogously the distance of the $n^{\text{th}}$ query in sequential BO to the last $q-1$ sequential queries as
\begin{align}
    \Delta_{seq} \equiv \min\limits_{i \in [q-1]} \lVert x_{n} - x_{n-i} \rVert_2.
\end{align}

In \Cref{fig:dist}, we show the median and inter-quartile range of these distances for a number of test functions, across 20 independent runs. It can be seen that existing methods query further from busy locations than standard acquisition, as is intended. But, contrary to \Cref{hyp:claim}, standard acquisition does not systematically perform repeated queries from the very start of the optimization. In fact, the standard UCB and, in some cases, also LogEI, exhibit the desirable transition from an initial exploratory phase ($\Delta$~large) to an exploitation stage ($\Delta$~small).  

Perhaps even more surprising is the comparison of standard acquisition functions to their respective sequential counterparts. \Cref{fig:dist} shows that standard acquisition in the asynchronous setting very closely matches the query distances of the sequential one. Interestingly, this analysis reveals that the lack of an exploitation phase of LogEI on, e.g., Hartmann is not due to issues described in \Cref{hyp:claim}, but simply a characteristic of the LogEI itself. 

This analysis further supports the conceptual arguments made in \Cref{sec:concept}. Existing purpose-built methods enforcing diversity seem to suffer from over-exploration, subsequently leading to a weaker exploitation phase of the optimization. On the other hand, standard acquisition performs an exploration-exploitation trade-off closely aligned with that of the optimally informed sequential BO.

\section{CONCLUSION}
\label{sec:concl}

A recent line of research falls under the overarching theme: \textit{BO is easier than previously thought and simple approaches work surprisingly well} \citep{hvarfner2024vanilla,xu2024standard,ament2023unexpected}. In this work, we add to this by demonstrating that simple and standard acquisition functions perform well in asynchronous BO, while existing methods for asynchronous BO are founded on the unsupported assumption in \Cref{hyp:claim}. 

While the literature has not presented a purpose-built method superior to standard acquisition in asynchronous BO, this does not mean such a method cannot exist. 
However, we have shown that existing methods in fact often degrade performance by enforcing query-diversity, likely by inducing an inefficient over-exploration of the search space. 

While further exploration may often be needed to find the global optimum, it is currently unclear how to conduct this efficiently. 
In future research, we plan to further explore mechanisms taking into account the busy locations to achieve efficient exploration in the asynchronous BO algorithm. 

\clearpage

\subsubsection*{Acknowledgements}
We thank Tommy Rochussen for helpful discussions.
VF was supported by the Branco Weiss Fellowship.

\bibliographystyle{plainnat}  %
\bibliography{references} 

\clearpage
\section*{Checklist}

\begin{enumerate}

  \item For all models and algorithms presented, check if you include:
  \begin{enumerate}
    \item A clear description of the mathematical setting, assumptions, algorithm, and/or model. Yes. See \Cref{alg:asyncBO} and \Cref{sec:app_acqf}.
    \item An analysis of the properties and complexity (time, space, sample size) of any algorithm. Yes. See \Cref{sec:theoretical}.
    \item (Optional) Anonymized source code, with specification of all dependencies, including external libraries. Yes. \url{https://github.com/fortuinlab/SAIS-BO}
  \end{enumerate}

  \item For any theoretical claim, check if you include:
  \begin{enumerate}
    \item Statements of the full set of assumptions of all theoretical results. Yes. See \Cref{sec:app_bound}.
    \item Complete proofs of all theoretical results. Yes. See \Cref{sec:app_bound}.
    \item Clear explanations of any assumptions. Yes. See \Cref{sec:app_bound}.    
  \end{enumerate}

  \item For all figures and tables that present empirical results, check if you include:
  \begin{enumerate}
    \item The code, data, and instructions needed to reproduce the main experimental results (either in the supplemental material or as a URL). Yes. \url{https://github.com/fortuinlab/SAIS-BO}
    \item All the training details (e.g., data splits, hyperparameters, how they were chosen). Yes. See \Cref{sec:app_impl}.
    \item A clear definition of the specific measure or statistics and error bars (e.g., with respect to the random seed after running experiments multiple times). Yes. See \Cref{sec:empirical}.
    \item A description of the computing infrastructure used. (e.g., type of GPUs, internal cluster, or cloud provider). Yes. See \Cref{sec:app_comp}.
  \end{enumerate}

  \item If you are using existing assets (e.g., code, data, models) or curating/releasing new assets, check if you include:
  \begin{enumerate}
    \item Citations of the creator If your work uses existing assets. Yes. See \Cref{sec:app_impl,sec:app_exper}.
    \item The license information of the assets, if applicable. Yes. See \Cref{sec:app_impl,sec:app_exper} and \url{https://github.com/fortuinlab/SAIS-BO}.
    \item New assets either in the supplemental material or as a URL, if applicable. Yes. \url{https://github.com/fortuinlab/SAIS-BO}
    \item Information about consent from data providers/curators. Not Applicable.
    \item Discussion of sensible content if applicable, e.g., personally identifiable information or offensive content. Not Applicable.
  \end{enumerate}

  \item If you used crowdsourcing or conducted research with human subjects, check if you include:
  \begin{enumerate}
    \item The full text of instructions given to participants and screenshots. Not Applicable.
    \item Descriptions of potential participant risks, with links to Institutional Review Board (IRB) approvals if applicable. Not Applicable.
    \item The estimated hourly wage paid to participants and the total amount spent on participant compensation. Not Applicable.
  \end{enumerate}

\end{enumerate}

\clearpage
\appendix
\thispagestyle{empty}

\onecolumn

\counterwithin{figure}{section}
\counterwithin{table}{section}
\renewcommand{\thefigure}{\thesection.\arabic{figure}}
\renewcommand{\thetable}{\thesection.\arabic{table}}

\section{Acquisition function details}
\label[appendix]{sec:app_acqf}

The probabilistic nature of the surrogate model allows the acquisition function to reason about uncertainty in $f(\cdot)$. This, in turn, guides the tradeoff of exploration and exploitation on the search space $\mathcal{X}$. Formally, the next query location, $x'$, is proposed by an acquisition rule, with acquisition function $\alpha: \mathcal{X} \mapsto \mathbb{R}$, as
\begin{align}
    x' = \arg\max\limits_{x \in \mathcal{X}} \ \alpha(x\mid  D_n).
\end{align}
We drop the conditioning on $D_n$ for notational convenience, unless explicitly mentioning it is necessary. This section presents a number of standard and asynchronous acquisition rules relevant to our work.

\subsection{Standard acquisition}
\label[appendix]{sec:std_acqf}

\paragraph{Upper Confidence Bound} The GP Upper Confidence Bound (UCB) \citep{srinivas2010gaussian} is a simple heuristic based on optimizing a quantile of the credibility interval, for example, the 95\% outcome. It is defined as
\begin{align}
    \alpha_{UCB}(x) = \mu(x\mid D_n) + \sqrt{\beta}\sigma(x \mid D_n),
\end{align}
with $\mu(\cdot)$ and $\sigma(\cdot)$ as in \Cref{equ:post_mean,equ:post_var}.

It can be seen as a weighted sum of the posterior mean and standard deviation, where the relative contribution of each summand is set via the hyperparameter $\beta$. The exploration-exploitation tradeoff is controlled by $\beta$, which may be set to a fixed value or according to some schedule \citep{srinivas2010gaussian}. In our work we keep $\beta=2$ fixed.

\paragraph{Expected Improvement} The Expected Improvement (EI) \citep{movckus1975bayesian, jones1998efficient} acquisition function assigns utility to an input location, $x \in \mathcal{X}$, according to how much the associated function value, $f(x)$, is expected (under the surrogate posterior) to improve on the best function value, $y^*_n = \max\limits_{i\in [n] } \mathbf{y}$, observed so far. Formally, 
\begin{align}
    \alpha_{EI}(x) = \mathbb{E}_{f \mid D_n}[\max(f(x) - y^*_n, 0)].
\end{align}
For numerical stability in the optimization, this work uses the natural logarithm of EI \citep{ament2023unexpected}.

\paragraph{Random Search} While there are many ways to perform BO with random search acquisition, the simplest variant sampling
\begin{align}
    x' \sim \mathcal{U}[0,1]^d
\end{align}
is used in this work. A more elaborate scheme would, e.g., be to sample from a low-discrepancy quasi-random sequence, such as the Halton sequence \citep{halton1964algorithm}.

\subsection{Asynchronous acquisition rules}
\label[appendix]{sec:app_incumb}

\paragraph{Monte Carlo sampling}\citet{ginsbourger2011dealing} present a \textit{Monte Carlo} sampling-based estimate of the expected  EI, where the expectation is with respect to the unknown function values, $\mathbf{y}_b$, under the GP surrogate. In particular, they form 
\begin{align}
    \alpha_{E-EI}(x \mid \mathcal{B}) = \frac{1}{N}\sum_{i=1}^N\alpha_{EI}(x \mid D_n, (\mathcal{B}, \mathbf{y}_{b,i})),
\end{align}
where $\{\mathbf{y}_{b,i}\}_{i=1}^N$ are i.i.d. samples from the GP surrogate posterior predictive
\begin{align}
    p(\mathbf{y}_b\mid D_n) =  \mathcal{N}(\mathbf{y}_b\mid \mu_b, \Sigma_b),
\end{align}
with
\begin{align}
    \mu_b &= K_{\mathcal{B} X}(K_{XX} + \eta^2\mathbf{I}_n)^{-1}\mathbf{y}  \in \mathbb{R}^{q-1} \label{equ:batch_mean}\\
    \Sigma_b &=K_{\mathcal{B} \mathcal{B}} - K_{\mathcal{B} X}(K_{XX} + \eta^2\mathbf{I}_n)^{-1}K_{X\mathcal{B}} \in \mathbb{R}^{(q-1)\times(q-1)}.
\end{align}
While \citet{ginsbourger2011dealing} do this for the EI, this approach may be taken for any analytical acquisition function.

\paragraph{Hallucination} A further approach to dealing with the unobserved function values is \textit{hallucination}-based methods, such as the Kriging Believer \citep{ginsbourger2010kriging}. Here, the unknown values, $\mathbf{y}_b$, are simply replaced by their posterior means under the GP surrogate when forming the acquisition function. In particular, for any analytic acquisition function, $\alpha(\cdot)$,
\begin{align}
    \alpha_{KB}(x \mid \mathcal{B}) = \alpha(x \mid D_n, (\mathcal{B },\mu_b)),
\end{align}
with $\mu_b$ as in \Cref{equ:batch_mean}.

While developed for synchronous $q$-batch construction, the constant liar heuristic offers an additional mechanism to account for locations under evaluation \citep{ginsbourger2007multi}.

\paragraph{Thompson sampling} Thompson sampling (TS) is an acquisition rule based on sampling the surrogate posterior \citep{thompson1933likelihood}. The goal of the approach is to sample points in the input space where the maximum is most likely to be. 
Consider a function $g \sim p(\cdot \mid D_n)$ which is sampled from the surrogate posterior. The probability that a particular input is the maximizer, $x^*$, is given by
\begin{align}
    p(x^*|D_n) &= \int p(x^*\mid g) \ p(g \mid D_n) dg \\
    &=\int \delta_{\arg \max \limits_{x \in \mathcal{X}} g(x)}(x^*) \ p(g \mid D_n) dg,  
\end{align}
we can sample from $p(x^*|D_n)$ in a simple two step procedure. First, draw $g \sim p(\cdot \mid D_n)$, and then return $x'= \arg \max \limits_{x \in \mathcal{X}} \ g(x)$.

In this work, all Thompson samples are drawn using the state-of-the-art decoupled sampling method by \citet{wilson2020efficiently}, giving differentiable posterior function samples. This significantly enhances performance over the standard approach of sampling function values at a discrete set of inputs.

In the context of asynchronous batch BO, this method was proposed, analyzed theoretically, and evaluated empirically by \citet{kandasamy2018parallelised}. In line with Hypothesis \ref{hyp:claim}, they motivate the need for their method with the occurrence of redundant function evaluations in standard asynchronous BO, while acknowledging that queries will not be exactly repeated. In the one real-world experiment performed by \citet{kandasamy2018parallelised}, this method did not outperform standard EI. Moreover, this work does not use the numerically advantageous LogEI, since it precedes the work by \citet{ament2023unexpected}. This likely explains the inferior performance of standard EI in their work compared to ours.

\paragraph{AEGIS} Proposed by \citet{de2021asynchronous}, Asynchronous $\epsilon$-Greedy Global Search (AEGIS) is an acquisition rule probabilistically combining three heuristics: \textit{(i)} performs TS with probability $\epsilon_T$, \textit{(ii)} chooses randomly from the Pareto Frontier for the two objectives of maximizing the mean ($\mu(\cdot)$) and the variance ($\sigma^2(\cdot)$) with probability $\epsilon_P$, and \textit{(iii)} otherwise optimizes the surrogate mean. The probabilities for the different modes are set as 
\begin{align}
    \epsilon_T &= \epsilon_P = \nicefrac{\epsilon}{2} \\
    \epsilon &= \min\{\nicefrac{2}{\sqrt{d}}, 1\},  
\end{align}
such that the tendency to exploit the surrogate mean decays as $\nicefrac{1}{\sqrt{d}}$.

\citet{de2021asynchronous} are motivated by Hypothesis \ref{hyp:claim} and the shortcomings of TS.  They choose not to compare their method to standard acquisition rules like EI or UCB.

\paragraph{(Local) Lipschitz penalization}

The asynchronous Lipschitz penalization (LP) method, called "PLAyBOOK" by \citet{alvi2019asynchronous}, aims to ensure acquisition diversity by creating exclusion cones in the acquisition surface, centered on the busy locations, $\mathcal{B}$. This is based on work for sequentially constructing a $q$-batch, in synchronous batch BO \citep{gonzalez2016batch} (\Cref{fig:parBO}). The extent of the penalization depends on an estimate of the objective function's Lipschitz constant, $L$, and the global optimum, $f^*$. Formally, \citet{alvi2019asynchronous} design the local penalizer centered at busy location $x_j$, $\varphi: \mathcal{X} \mapsto [0,1]$, as
\begin{align}
\label{equ:penal}
    \varphi(x \mid x_j) = \min\bigg\{ \frac{\hat{L}\| x-x_j\|}{\mid \mu(x_j)-y^*_n\mid + \ \gamma \sigma(x_j)}, 1 \bigg\}.
\end{align}
The objective optimum is estimated as the best function value found so far, $y^*_n = \max\limits_{i \in [n] } \mathbf{y}$. The Lipschitz constant is estimated from the gradient of the posterior mean as $\hat{L} = \max\limits_{x \in \mathcal{X}} \nabla \mu(x)$. Any analytic acquisition function, $\alpha(\cdot)$, may then be locally penalized at a set of busy locations $\mathcal{B}~=~\{x_j\}_{j=1}^{q-1}$ as
\begin{align}
\label{equ:app_LP}
    \alpha_{LP}(x\mid \mathcal{B}) = \alpha_{UCB}(x) \prod_{j=1}^{q-1}\varphi(x \mid x_j).
\end{align}
While this is presented as an approximation,
\begin{align}
    \alpha_{LP}(x\mid \mathcal{B}) &\approx \mathbb{E}[\alpha_{UCB}(x \mid D_n,  D_b)\mid D_n, \mathcal{B}],
\end{align}
we show in \Cref{prop:ucb_KB} that in fact 
\begin{align}
    \mathbb{E}[\alpha_{UCB}(x \mid D_n,  D_b)\mid D_n, \mathcal{B}] = \alpha_{KB}(x \mid \mathcal{B}).
\end{align}
Judging from \Cref{fig:iter,fig:app_iter}, the quality of this approximation, $\alpha_{LP}(\cdot \mid \mathcal{B}) \approx \alpha_{KB}(\cdot \mid \mathcal{B})$, is questionable, albeit both $\alpha_{LP}(\cdot \mid \mathcal{B})$ and $\alpha_{KB}(\cdot \mid \mathcal{B})$ seem to have a common optimum.

Additionally, \citet{alvi2019asynchronous} present a version of this where the Lipschitz constant is not shared by all $q-1$ penalizers, but estimated locally around the respective busy location. The search spaces $\{\mathcal{X}_j\}_{j=1}^{q-1}$ for the local Lipschitz (LLP) estimation are then defined through the kernel lengthscales. In particular, $\mathcal{X}_j \subset \mathcal{X}$ is a hyper rectangle centered on $x_j$, with side lengths equal to the lengthscales of the respective dimensions.

Following Hypothesis \ref{hyp:claim}, they motivate their method with the danger of repeated and redundant queries at and in the vicinity of busy locations, $\mathcal{B}$. In their work, they do not compare their method to standard acquisition rules like EI or UCB.

From \Cref{fig:hpy} it can be seen that these methods (LP-UCB and LLP-UCB) learn relatively long lengthscales. This can be expected to result in rather small estimated Lipschitz constants, $\hat{L}$, which in turn results in penalization of large regions around the busy locations. This is then likely the driver of the over-exploration observed for these methods (see, e.g., \Cref{fig:2d}).

\section{Formal statement and proof of \texorpdfstring{\Cref{thm:bound}}{}}
\label[appendix]{sec:app_bound}

\paragraph{Setting and notation} Let $\mathcal{X}\subset[0,1]^d$ be compact and the objective $f\sim\mathcal{GP}(0,k_{\phi})$ with $k_{\phi}(x,x)\le1$. Given a finite subset $A = \{x_1, x_2,  ..., x_n\} \subset \mathcal{X}$, we have $(f_A)_i = f(x_i)$ and $(\epsilon_A)_i \sim \mathcal{N}(0, \eta^2)$. Observations are $y_A = f_A + \epsilon_A \in \mathbb{R}^n$. In the following, the index $j$ refers to the $j^{th}$ acquisition being made by \Cref{alg:asyncBO}. The set of \emph{completed} observations $(x, y)$ available at step $j$ is denoted by $\mathcal{D}_j$.\footnote{Note the difference to $D_n$, a set of $n$ completed evaluations (cf. \Cref{sec:prelim}).} In line with this, $\mu_{\mathcal{D}_j}$ and $\sigma_{\mathcal{D}_j}$ denote the GP posterior mean and standard deviation constructed from data $\mathcal{D}_j$. At step $j$, with completed set $\mathcal{D}_j$, our policy asynchronous 
UCB selects
\begin{align} \label{equ:ucb}  
 x_j=\arg\max_{x\in\mathcal{X}} \ U_j(x),
 \end{align}
 with
 \begin{align}
 U_j(x)=\mu_{\mathcal{D}_j}(x)+\sqrt{\beta_j}\,\sigma_{\mathcal{D}_j}(x),
\end{align}
for a non-decreasing sequence $\beta_j$.

\paragraph{Relevant quantities} We introduce the quantity $\xi_q$, which was studied by \cite{Krause2008NearOptimal, Desautels2014Parallelizing} and also used by \cite{kandasamy2018parallelised}.
Informally, it places an upper bound on the information about $f$ in the $q-1$ pending evaluations at busy locations $\mathcal{B}$. Formally we have
\begin{align}
    \max _{\mathcal{B} \subset \mathcal{X},\left|\mathcal{B}\right|<q} I\left(f ; y_{\mathcal{B}} \mid {\mathcal{D}_n}\right) \leq \frac{1}{2} \log \left(\xi_q\right),
\end{align}
with $I$ the Shannon mutual information. Further, we make use of the maximum information gain (MIG) after $n$ evaluations, defined as 
\begin{align}
    \Psi_n  = \max _{A \subset \mathcal{X},\left|A\right|=n} I\left(f ; y_{A}\right).
\end{align}
A sequence $\beta_n$ satisfying the assumptions on which the relevant Lemmas rest is 
\begin{align}
\label{equ:beta}
\beta_j=4(d+1) \log (j)+2 d \log (d a b \sqrt{\pi}),
\end{align}
with $d$ the problem dimensionality and constants $a$ and $b$ from Assumption 8 in \cite{kandasamy2018parallelised}. Note that our choice of keeping $\beta_j=2$ fixed does not satisfy the conditions required for the theorem to hold, but seems to work well nonetheless (see \Cref{sec:std_acqf}).

Next, we state a formal version of \Cref{thm:bound}.
\begin{theorem}[Bound on BSR$(n)$ for asynchronous UCB]
\label{thm:bound_rig}
    Let $f \sim \mathcal{GP}(0, k_{\phi})$, where $k_{\phi}: \mathcal{X} \times \mathcal{X} \mapsto \mathbb{R}$ satisfies Assumption 8 in \citet{kandasamy2018parallelised} and, $w.l.o.g.$, $k_{\phi}(x,x) \le 1$. Then, for the UCB acquisition function used in \Cref{alg:asyncBO}, the Bayes simple regret after $n$ queries can be bounded as
    \begin{align}
        \textnormal{BSR}(n) \le \frac{\pi^2+\sqrt{2\pi}}{12n} + \sqrt{\frac{2 \xi_q \beta_n \Psi_n}{n\log(1+\eta^{-2})}}.
    \end{align} 
\end{theorem}
\begin{proof}
    Our proof heavily relies on \citet{kandasamy2018parallelised} and we will refer to relevant results of this work throughout. For a given number of steps, $n$, we seek a bound on the Bayes simple regret (BSR) defined as
\begin{align}
    \text{BSR}(n) = \mathbb{E}[f(x^*)- \max_{j \in [n]} f(x_j)],
\end{align}
where the expectation is w.r.t. the prior on $f$, the observation noises $\{{\epsilon_j}\}_{j=1}^n$.

We first define and deconstruct the Bayes cumulative regret (BCR, \cite{kandasamy2018parallelised}) as follows.
\begin{align}
    \text{BCR}(n) &= \sum_{j = 1}^n \mathbb{E}[f(x^*)- f(x_j)] \\
    &=\sum_{j = 1}^n \mathbb{E}[f(x^*) - f([x^*]_j) + f([x^*]_j) - U_j([x^*]_j)+U_j([x^*]_j)- f(x_j)] \label{equ:decomp} \\
   &= \underbrace{\sum_{j = 1}^n \mathbb{E}[f(x^*) - f([x^*]_j)]}_{A_1}
 + \underbrace{\sum_{j = 1}^n \mathbb{E}[f([x^*]_j) - U_j([x^*]_j)]}_{A_2}
 + \underbrace{\sum_{j = 1}^n \mathbb{E}[U_j([x^*]_j) - f(x_j)]}_{A_3}
\end{align}
In order to make use of relevant results in \cite{kandasamy2018parallelised}, we adopt the construction of a finite data-independent grid, $\nu_j \subset \mathcal{X}$, at each step $j$. Using this, we define $[\cdot]_j: \mathcal{X} \mapsto \nu_j$, such that $[x]_j$ is the closest point to $x$ in $\nu_j$. Next, $A_1, A_2$ and $A_3$ are bounded separately, where the bounds on $A_1$ and $A_2$ are directly from \citet{kandasamy2018parallelised}. For $A_1$ we have
\begin{align}
    A_1 &\le \sum_{j = 1}^n \mathbb{E}[|f(x^*) - f([x^*]_j)|] \le \sum_{j = 1}^n \frac{1}{2j^2} \le \frac{\pi^2}{12},
\end{align}
where the second inequality follows from Lemma 12 of \cite{kandasamy2018parallelised}. For $A_2$ we have
\begin{align}
    A_2 &\leq \ \mathbb{E}\Big[ \sum_{j=1}^n \mathbf{1}\{ f([x^*]_j) > U_j([x^*]_j) \} \cdot ( f([x^*]_j) - U_j([x^*]_j) )\Big] \\
    &\leq \sum_{j=1}^n \sum_{x \in \nu_j} \mathbb{E}\big[ \mathbf{1}\{ f(x) > U_j(x) \} \cdot ( f(x) - U_j(x) ) \big]
    \leq \sum_{j=1}^n \sum_{x \in \nu_j} \frac{1}{j^2 \sqrt{2 \pi} |\nu_j|} = \frac{\sqrt{2 \pi}}{12},
\end{align}
where we sum only the positive terms in the first step, bound each term $j$ by the sum of corresponding terms on the grid $\nu_j$ and then apply Lemma 13 from \cite{kandasamy2018parallelised}. The final equality follows from the cardinality of $\nu_j$, as per its definition in \cite{kandasamy2018parallelised}.

Our key contribution is the decomposition in \Cref{equ:decomp} and the following bound on the resulting $A_3$. We make this decomposition to invoke \Cref{equ:ucb}, which defines our proposed asynchronous BO policy. We first bound 
\begin{align}
    A_3 &\le \sum_{j = 1}^n \mathbb{E}[U_j(x_j) - f(x_j)]\\
        &=: A'_3,
\end{align}
since by the definition of the acquisition rule $U_j([x^*]_j) \le U_j(x_j)$, regardless of the value $x^*$ (see \Cref{equ:ucb}). We proceed by first simplifying and then bounding $A'_3$.
\begin{align}
    A'_3 &= \sum_{j = 1}^n \mathbb{E}[\mathbb{E}[U_j(x_j) - f(x_j)|\mathcal{D}_j]] = \sum_{j = 1}^n \mathbb{E}[U_j(x_j) -\mathbb{E}[f(x_j)|\mathcal{D}_j]] \\
    &= \sum_{j = 1}^n \mathbb{E}[\mu_{\mathcal{D}_j}(x_j)+\sqrt{\beta_j}\,\sigma_{\mathcal{D}_j}(x_j) - \mu_{\mathcal{D}_j}(x_j)] = \sum_{j = 1}^n \mathbb{E}[\sqrt{\beta_j}\,\sigma_{\mathcal{D}_j}(x_j)] \\
    &\le \sqrt{\beta_n} \sum_{j = 1}^n \mathbb{E}[\sigma_{\mathcal{D}_j}(x_j)] \\
    &\le \sqrt{\beta_n} \ \sqrt{\xi_q} \sum_{j = 1}^n \mathbb{E}[\sigma_{j-1}(x_j)] \\
    &\le \sqrt{\beta_n} \ \sqrt{\xi_q} \mathbb{E}\Big[\Big(n \sum_{j = 1}^n \sigma^2_{j-1}(x_j)\Big)^{\nicefrac{1}{2}}\Big] \\
    &\le \sqrt{\frac{2 \xi_q n \beta_n \Psi_n}{\log(1+\eta^{-2})}}
\end{align}
where we (in the following order) use that $U_j$ and thus $x_j$ are fixed given $\mathcal{D}_j$,  $\beta_j$ is non-decreasing, Lemma 10 from \cite{kandasamy2018parallelised}, Cauchy-Schwarz and finally Lemma 7 from \cite{kandasamy2018parallelised}. 

Using the bounds on $A_1$, $A_2$ and $A_3$ we get
\begin{align}
    \text{BCR}(n) \le \frac{\pi^2+\sqrt{2\pi}}{12} + \sqrt{\frac{2 \xi_q n \beta_n \Psi_n}{\log(1+\eta^{-2})}}
\end{align}
from which we can then construct the bound 
\begin{align}
    \text{BSR}(n) \le \frac{\pi^2+\sqrt{2\pi}}{12n} + \sqrt{\frac{2 \xi_q \beta_n \Psi_n}{n\log(1+\eta^{-2})}} = \frac{C_1}{n} + \sqrt{\frac{C_2 \xi_q \beta_n \Psi_n}{n}}
\end{align}
using BSR$(n) \le \frac{1}{n}$BCR$(n)$ from, e.g., \citep{kandasamy2018parallelised}. We define $C_1 \equiv \frac{\pi^2+\sqrt{2\pi}}{12}$ and $C_2 \equiv \frac{2 }{\log(1+\eta^{-2})}$, which are independent of both $q$ and $n$.
\end{proof}

The informal statement in \Cref{thm:bound} follows by realizing that the sequence $\beta_j$ in \Cref{equ:beta} satisfies $\beta_n \in \mathcal{O}(\log(n))$ and $\frac{1}{n} \in o\Big(\sqrt{\frac{\log(n)}{n}}\Big)$.

For completeness, we state the corresponding bounds on the BSR$(n)$ under the sequential and asynchronous TS policies, as shown in \citet{kandasamy2018parallelised}. 

\begin{theorem}[Bound on BSR$(n)$ for sequential TS (Corollary 15, \citet{kandasamy2018parallelised})]
\label{thm:bound_seqTS}
    Let $f \sim \mathcal{GP}(0, k_{\phi})$, where $k_{\phi}: \mathcal{X} \times \mathcal{X} \mapsto \mathbb{R}$ satisfies Assumption 8 in \citep{kandasamy2018parallelised} and, $w.l.o.g.$, $k_{\phi}(x,x) \le 1$. Then, for sequential TS the Bayes simple regret after $n$ queries can be bounded as
    \begin{align}
        \textnormal{BSR}(n) \le \frac{2\pi^2+\sqrt{2}\pi^{\nicefrac{5}{2}}}{12n} + \sqrt{\frac{2 \beta_n \Psi_n}{n\log(1+\eta^{-2})}}.
    \end{align} 
\end{theorem}

\begin{theorem}[Bound on BSR$(n)$ for asynchronous TS (Theorem 14, \citet{kandasamy2018parallelised})]
\label{thm:bound_TS}
    Let $f \sim \mathcal{GP}(0, k_{\phi})$, where $k_{\phi}: \mathcal{X} \times \mathcal{X} \mapsto \mathbb{R}$ satisfies Assumption 8 in \citep{kandasamy2018parallelised} and, $w.l.o.g.$, $k_{\phi}(x,x) \le 1$. Then, for the TS acquisition method used in \Cref{alg:asyncBO}, the Bayes simple regret after $n$ queries can be bounded as
    \begin{align}
        \textnormal{BSR}(n) \le \frac{2\pi^2+\sqrt{2}\pi^{\nicefrac{5}{2}}}{12n} + \sqrt{\frac{2 \xi_q \beta_n \Psi_n}{n\log(1+\eta^{-2})}}.
    \end{align} 
\end{theorem}

Noting that the bound in \Cref{thm:bound_TS} is equal in all relevant aspects to our bound in \Cref{thm:bound_rig}, all downstream conclusions apply equally to both bounds. In particular, we can conclude that under an uncertainty sampling initialization \citep{Desautels2014Parallelizing, Krause2008NearOptimal,kandasamy2018parallelised}, the quantity $\xi_q$ can be bounded by a small kernel-dependent constant, independent of $n$. We refer to \citet{kandasamy2018parallelised} for further details on this.  Finally, we can thus conclude that, under the uncertainty sampling initialization, standard asynchronous UCB is almost as good as sequential Thompson sampling. This immediately follows from Corollary 4 of \citet{kandasamy2018parallelised} and the equivalence of the bounds in \Cref{thm:bound_rig,thm:bound_TS}.

\section{Proof \texorpdfstring{\Cref{prop:ucb_KB}}{}}
\label[appendix]{sec:app_proof}

We begin by restating \Cref{prop:ucb_KB}.
\begin{proposition*}[Marginalized UCB is the Kriging Believer]

Consider the random Upper Confidence Bound, $\alpha_{UCB}(x|D_n, (\mathcal{B}, \mathbf{y}_b))$, of the GP surrogate posterior, with unobserved function values, $\mathbf{y}_b$, at known busy locations, $\mathcal{B}$. Then it holds that
\begin{align*}
    \mathbb{E}[\alpha_{UCB}(x \mid D_n,  (\mathcal{B}, \mathbf{y}_b))\mid D_n, \mathcal{B}] = \alpha_{UCB}(x \mid D_n, (\mathcal{B}, \mu_b)).
\end{align*}
\end{proposition*}

\begin{proof}
    
\begin{align*}
    \mathbb{E}[\alpha_{UCB}(x \mid D_n,  D_b)\mid D_n, \mathcal{B}]
    &= \int \big[\mu(x \mid D_n, D_b) + \sqrt{\beta} \sigma(x\mid D_n,  D_b) \big] \  p(\mathbf{y}_b\mid D_n) \mathrm{d}\mathbf{y}_b\\ 
    &=\int \mu(x \mid D_n, D_b) \  p(\mathbf{y}_b\mid D_n) d\mathbf{y}_b \ + \ \int  \sqrt{\beta}\sigma(x\mid D_n,  D_b) \  p(\mathbf{y}_b\mid D_n) \mathrm{d}\mathbf{y}_b \\
\end{align*}
We note that the value for $\sigma(x| D_n, D_b)$ (as given in Equation~\eqref{equ:post_var}) does not depend on the as-of-yet unknown outputs. This makes the second expectation trivial, allowing the objective to be written as 
\begin{align*}
    \mathbb{E}[\alpha_{UCB}(x \mid D_n,  D_b)\mid D_n, \mathcal{B}]
    &= \int k(x, X\cup\mathcal{B})\Bigg[\begin{bmatrix}
K_{XX} & K_{X\mathcal{B}} \\
K_{\mathcal{B}X} & K_{\mathcal{B}\mathcal{B}}
\end{bmatrix} + \eta^2\mathbf{I}_{n+q-1}) \Bigg]^{-1}\begin{bmatrix}
\mathbf{y} \\
\mathbf{y}_b
\end{bmatrix} \  p(\mathbf{y}_b\mid D_n) \mathrm{d} \mathbf{y}_b \\
&\ \ \ \ \ \ + \sqrt{\beta} \sigma(x \mid X, \mathcal{B}) \\
\end{align*}
The only random variable in the above integral is $\mathbf{y}_b$, which has known expectation $\mu_b$. This integral therefore exactly recovers the prediction if we assume that we will observe the mean value of the GP for the currently busy locations 
\begin{align*}
    \mathbb{E}[\alpha_{UCB}(x \mid D_n,  D_b)\mid D_n, \mathcal{B}] &=\mu(x\mid D_n, (\mathcal{B}, \mu_b)) + \sqrt{\beta} \sigma(x \mid X, \mathcal{B}) \\
    &\\
    &=\alpha_{UCB}(x \mid D_n, (\mathcal{B}, \mu_b)). 
\end{align*}
\end{proof}
We note that this is equivalent to the KB heuristic \citep{ginsbourger2010kriging}.

\section{Implementation details}
\label[appendix]{sec:app_impl}

We implement our optimization pipeline in BoTorch \citep{balandat2020botorch} and GPyTorch \citep{gardner2018gpytorch}, and make the code available at \url{https://github.com/fortuinlab/SAIS-BO}. A zero-mean Gaussian Process prior with ARD-RBF kernel is used to form the surrogate in all experiments. The kernel hyperparameters, $\phi$,  and the observation noise, $\eta^2$,  were fit to optimize the posterior $p(\phi|\mathbf{y})$ \citep{williams2006gaussian}. We use the dimensionality-scaled hyper-prior on lengthscales proposed by \citet{hvarfner2024vanilla}, which is the default in BoTorch. All inputs were normalized to the unit hypercube $[0,1]^d$, and function values were standardized to have zero mean and unit variance.

The optimization of the acquisition function, as well as that of marginal likelihood, was carried out using a multi-restart strategy with the L-BFGS-B algorithm \citep{byrd1995limited}. Following standard practice, we optimize the best $10$ from an initial set of $1000d$ candidates. We deem this to be sufficient to exclude the optimization of the acquisition function as a source of the observed effects.

We kickstart the optimization with $3*d$ initial data points and then initialize the $q$ workers. The input locations of initial data, as well as the first batch of $q$ workers, are drawn from a randomly perturbed Halton sequence in the appropriate dimension \citep{halton1964algorithm}.

The Lipschitz penalizers of LP-UCB and LLP-UCB are approximated in a differentiable manner using $p = -5$, as suggested by \citet{alvi2019asynchronous}, and use $\gamma = 1$ (\Cref{sec:app_incumb}).

The implementation of AEGIS is adopted from the authors \citet{de2021asynchronous}, where the approximate Pareto front is found with code from \url{https://github.com/georgedeath/aegis/blob/main/aegis/batch/nsga2_pygo.py}.

We use $N=500$ \textit{Monte Carlo} samples to estimate the expected EI, since this is the BoTorch default (\url{https://botorch.org/docs/acquisition}).

\section{Toy experiments}
\label[appendix]{sec:app_toys}

A number of additional toy experiments are presented. 

\paragraph{1$-d$ task} \Cref{fig:app_iter} shows the continuation of \Cref{fig:iter}. It can be seen that the trend of the standard UCB and its purpose-built variants querying essentially the same locations, even though the UCB does not take into account the busy locations, persists. 

\paragraph{2$-d$ task}\Cref{fig:app_2d} shows additional query trajectories for standard acquisition, as well as \textit{penalization}- and \textit{randomness}-based methods. Rows 1 and 2, as well as 3 and 4, respectively, share initial data and initial worker locations. It can be seen that the standard methods (UCB, LogEI) gradually transition from an exploration to an exploitation phase and get close to the optimum with the iteration budget. Purpose-built methods do not display this transition from exploration to exploitation and do not generally find the optimum.

\begin{figure}[t]
    \centering

    \includegraphics{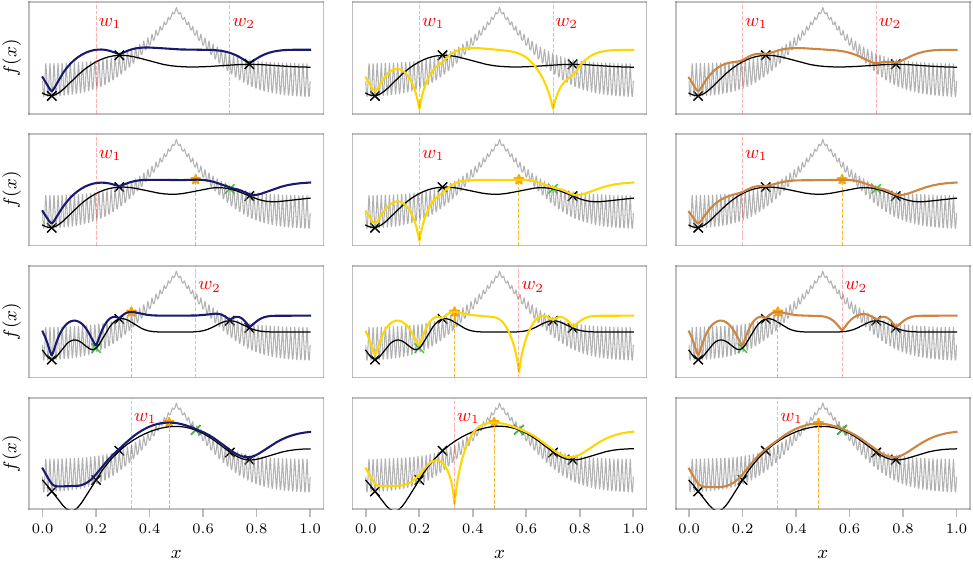}
    \includegraphics{figs/iter_legend.pdf}

    \caption{Continuation of \Cref{fig:iter}. It can be seen that the standard UCB queries essentially the same points as purpose-built methods, without taking into account the busy location. The update via the new sample provides sufficient information to prevent redundant or even repeated queries.}
    \label{fig:app_iter}
\end{figure}

\begin{figure}
    \centering

    \includegraphics{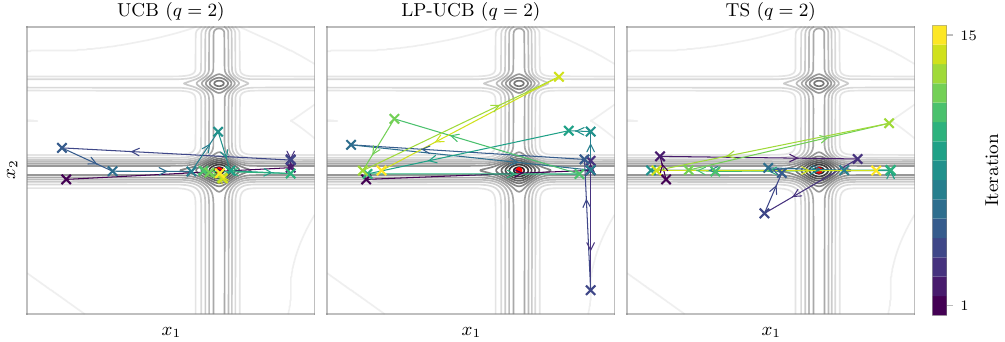}
    \includegraphics{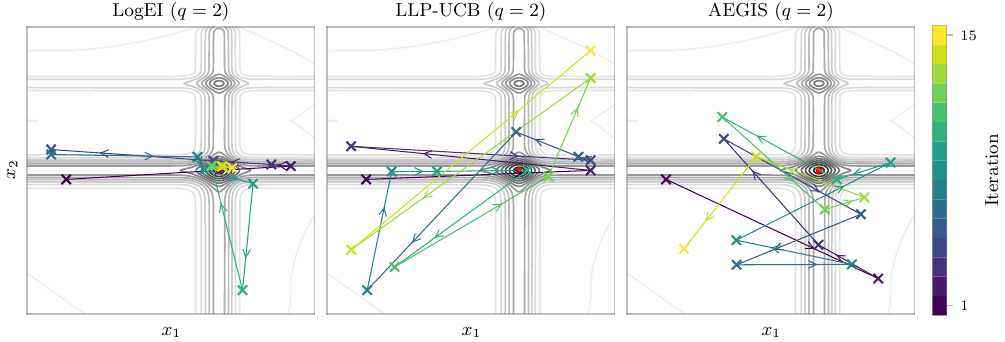}
    \includegraphics{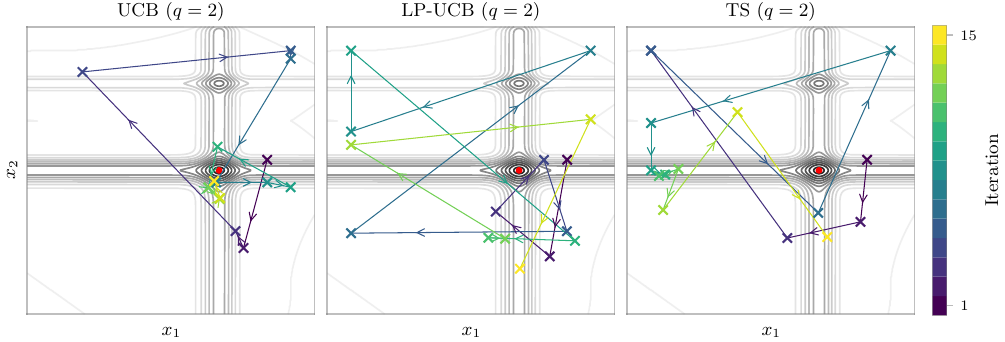}
    \includegraphics{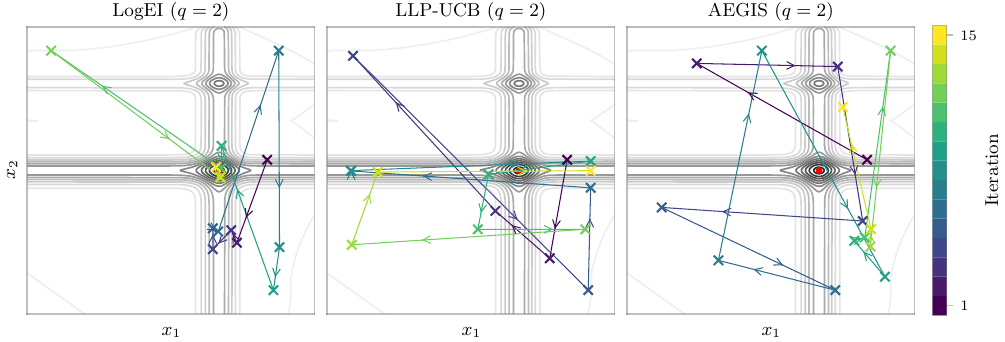}
    \includegraphics{figs/2D_legend.pdf}

    \caption{Query trajectories, see \Cref{sec:app_toys} for a detailed description.}
    \label{fig:app_2d}
\end{figure}

\section{Optimization tasks}
\label[appendix]{sec:app_exper}

\subsection{Synthetic test functions}
The noiseless synthetic test functions we use are available as part of the BoTorch package \citep{balandat2020botorch}. We present experiments in varying dimensions for the following functions: Ackley, Hartmann, Egg~Holder, Michalewicz, and Rosenbrock. Please refer to \url{https://botorch.readthedocs.io/en/latest/test_functions.html#module-botorch.test_functions.synthetic} or \url{https://www.sfu.ca/~ssurjano/optimization.html} for more details.

\subsection{Real-world tasks}
\label[appendix]{sec:real_tasks}
\paragraph{PROTEUS} PROTEUS is a coupled atmosphere-interior framework to simulate the temporal evolution of rocky planets \citep{Lichtenberg_2021_JGRP,Nicholls_2024_JGRP}. This deterministic forward simulator takes in an initial condition, $x$, returning observables $\gamma = PROTEUS(x)$. Given observables $\gamma_0$, one may want to infer the associated initial condition, i.e., perform a point-wise inversion of $PROTEUS(\cdot)$ at $\gamma_0$. We frame this inference as an optimization problem, with
\begin{align}
    \min\limits_{x \in \mathcal{X}} \|PROTEUS(x) - \gamma_0\|^2,
\end{align}
with known optimum $(x^*,0)$ for $x^*$ such that $PROTEUS(x^*) = \gamma_0$.
In this work, nine input variables are considered. Due to its adaptive time stepping and resulting varying evaluation times, PROTEUS naturally lends itself to asynchronous BO. PROTEUS is freely available, and we refer to \url{https://fwl-proteus.readthedocs.io/en/latest/} for detailed instructions on the installation.

\paragraph{Pest control} In this benchmark, the aim is to minimize the spread of pests and the (monetary) cost expended to this end \citep{oh2019combinatorial}. The design space consists of 25 categorical variables with five levels each ($\approx 2.98 \times 10^{17}$ combinatorial choices). This space represents 25 stations and the amount of pesticide used at each. We treat this combinatorial problem as an ordinal one, by discretizing the input space $[0,1]^{25}$ to $\{1,2,3,4,5\}^{25}$, thus creating a step function. This allows the categorical problem to be solved with standard continuous input methodology. The asynchronicity is provided by the varying cost throughout the input space, which we take to be the evaluation time. In order to compute the $\log(R)$, we take $-12$ to be the optimal value. The original code we use for this is from \url{https://github.com/yucenli/bnn-bo} and \url{https://github.com/QUVA-Lab/COMBO}.

\paragraph{XG-Boost} Hyperparameter optimization is a task every ML practitioner faces. The choice of hyperparameters often makes or breaks a model. For XG-Boost, we consider nine hyperparameters, which we tune to optimize the 5-fold cross-validation accuracy. This score is computed from the classification performance on the UCI Breast Cancer data \citep{wolberg1993breast}. The hyperparameters optimized are learning rate, number of boosting rounds (trees), maximum tree depth, minimum loss reduction to make a split, fraction of training examples to grow each tree on, fraction of features to use per tree, fraction of features to use per node (split), as well as the L1 and L2 regularization parameters on the leaf weights. The training time depends, e.g., on the maximum tree depth, resulting in the desired heterogeneous function evaluation times. The optimal value of the cross-validation accuracy is known to be $1$, allowing for computation of the $\log(R)$.

\paragraph{CNN} Arguably the most widely used data set in image classification, CIFAR10 \citep{krizhevsky2009learning} contains $60,000$ $32$x$32$ pixel images in ten classes. We randomly select a training and a validation set of $25,000$ examples each. In our experiment, we learn the hyperparameters of a 6-layer CNN pipeline to optimize validation accuracy after 20 epochs of training. This is framed as a  9-dimensional optimization problem of batch size, learning rate, momentum, and filter sizes for the six filters. At a fixed number of epochs, filter and batch size directly affect the train time of the CNN, creating varying evaluation times throughout the input space. The optimal value of the validation accuracy is known to be $1$, allowing for computation of the $\log(R)$.

\section{Additional experimental results}
\label[appendix]{sec:app_results}

\begin{figure}[t]
    \centering

    \includegraphics[width=\linewidth]{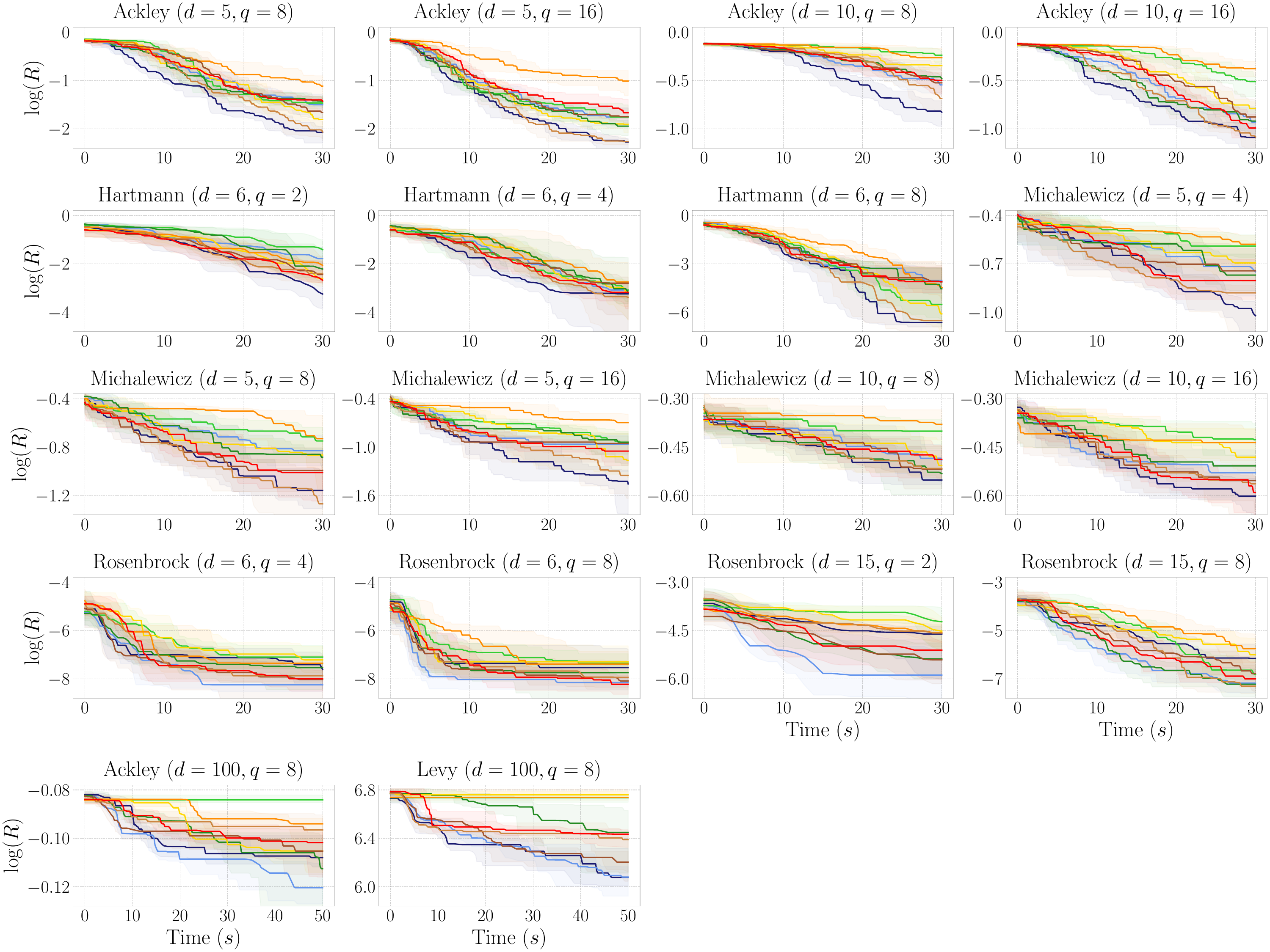}
    \includegraphics{figs/legend2.pdf}

    \caption{Standard acquisition functions are not significantly outperformed on any of our synthetic test functions. In fact, for our experiments, UCB often outperforms alternatives designed for asynchronous BO.}
    \label{fig:app_reg}
\end{figure}

\begin{figure}[t]
    \centering

    \includegraphics[width=0.32\linewidth]{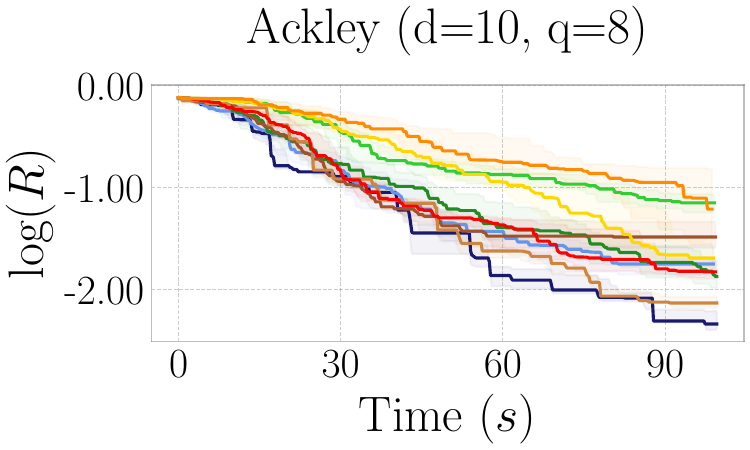}
    \includegraphics[width=0.32\linewidth]{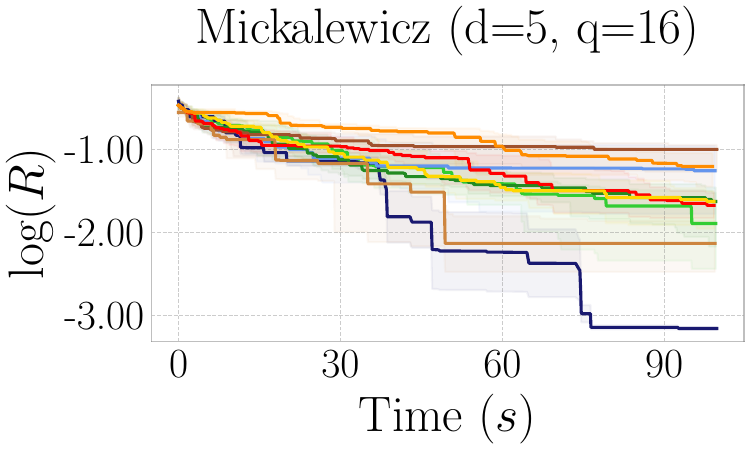}
    \includegraphics[width=0.32\linewidth]{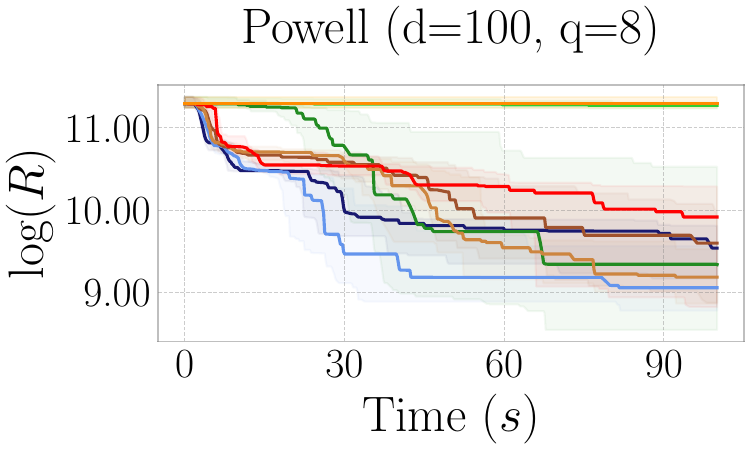}
    \includegraphics{figs/legend2.pdf}

    \caption{Extended run times show qualitatively the same results.}
    \label{fig:app_more_t}
\end{figure}

In \Cref{fig:app_reg,fig:app_dist,fig:app_dist_real,fig:app_dist_seq,fig:app_more_t}, we give additional experimental results on variants of the problems presented in \Cref{sec:empirical}. 

Additionally, we present a quantitative assessment in the form of a win-rate measure and a Mann-Whitney U-test. \Cref{tab:acc_wr,tab:mic_wr} show the win-rate results and should be read as "row-method" was better than "column-method" a fraction "entry" of the 20 trials. It can be seen that this is always larger than 0.5 for the standard UCB, and often even close to 1.0. \Cref{tab:acc_mwu,tab:mic_mwu} show the p-values of a Mann-Whitney U-test, with the null hypothesis: The regret values after $t=30$s of "row-method" and "column-method" follow the same distribution. As can be seen from the tiny p-values, the null can almost always be rejected at very small significance levels.

\begin{figure}[t]
    \centering

    \includegraphics[width=\linewidth]{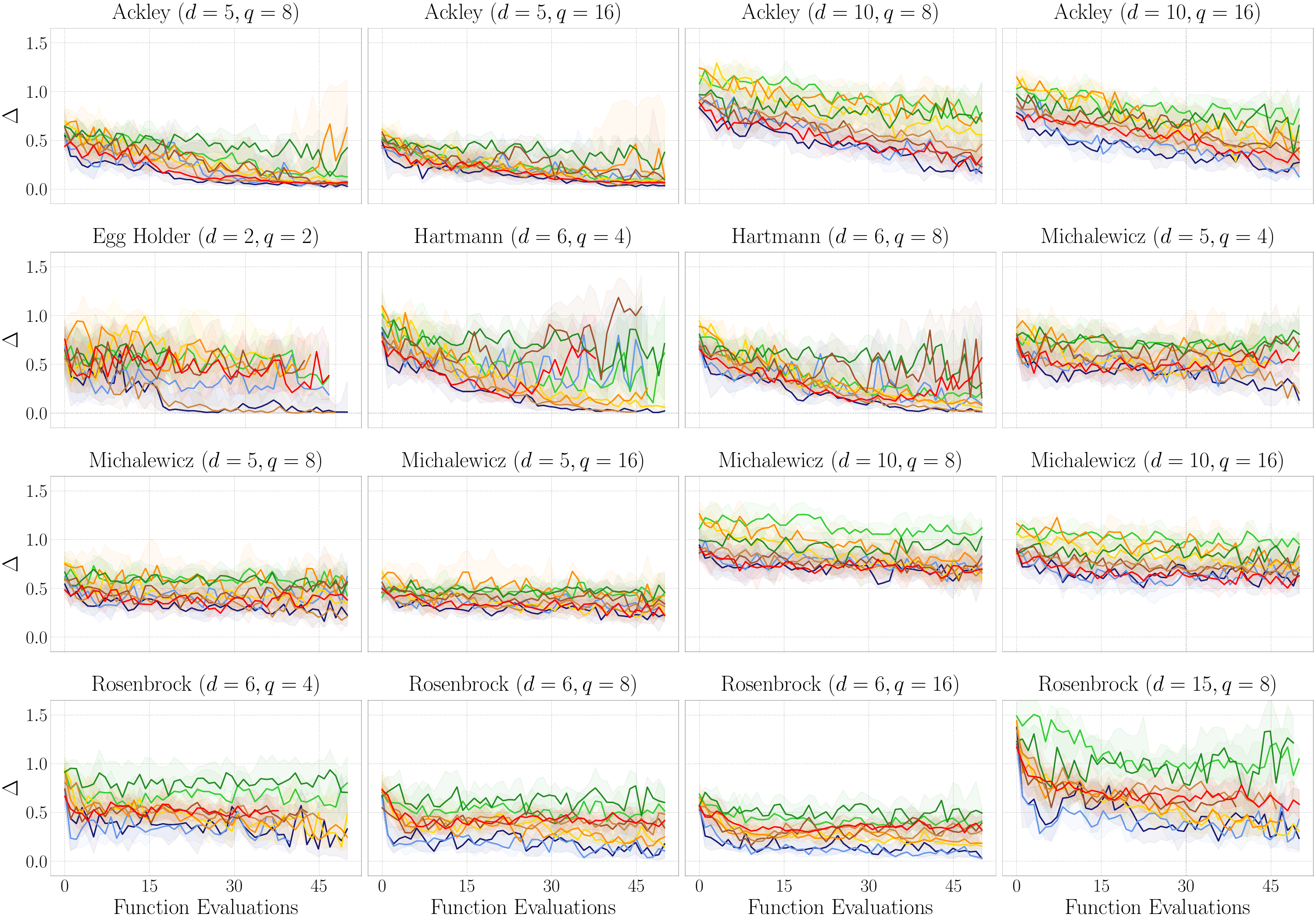}
    \includegraphics{figs/legend2.pdf}

    \caption{Standard acquisition functions query closer to busy locations than alternatives designed for asynchronous BO, but do not systemically repeat queries.}
    \label{fig:app_dist}
\end{figure}

\begin{figure}[t]
    \centering

    \includegraphics[width=\linewidth]{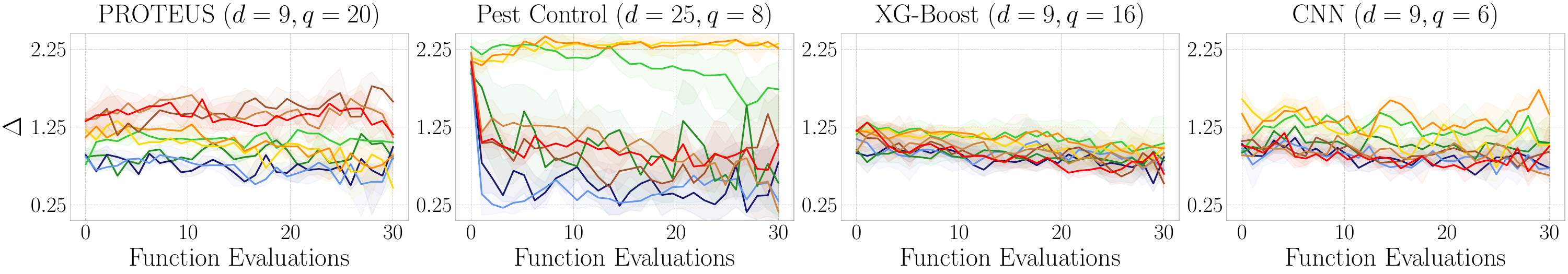}
    \includegraphics{figs/legend2.pdf}

    \caption{While standard acquisition queries the closest to busy locations, distances on real-world tasks are significantly larger than zero for all methods.}
    \label{fig:app_dist_real}
\end{figure}

\begin{figure}[t]
    \centering

    \includegraphics[width=\linewidth]{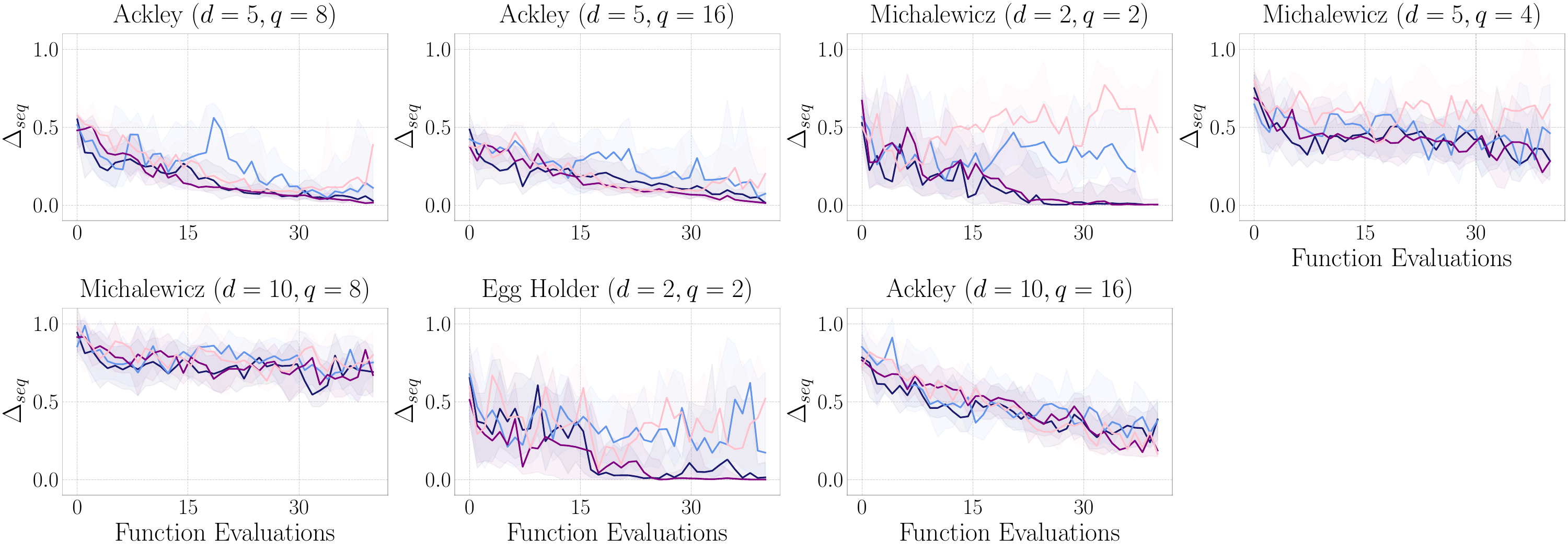}
    \includegraphics{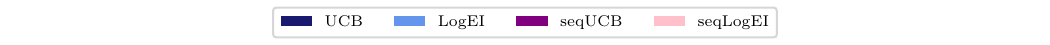}

    \caption{Standard acquisition functions mimic the distance profile of their respective sequential counterparts.}
    \label{fig:app_dist_seq}
\end{figure}

\begin{table}[H]
\centering
\caption{Ackley $(d=10, q=8)$, win-rate @ $t = 30$s}\label{tab:acc_wr}
\begin{tabular}{lcccccccc}
\hline
 & LogEI & TS & LP-UCB & LLP-UCB & AEGIS & KB-LogEI & KB-UCB & E-LogEI \\
\hline
UCB   & 0.85 & 1.0 & 0.95 & 1.0 & 1.0 & 0.95 & 0.75 & 0.9 \\
LogEI &  --  & 0.9 & 0.7  & 0.85 & 0.6 & 0.55 & 0.4  & 0.4 \\
\hline
\end{tabular}
\end{table}

\begin{table}[H]
\centering
\caption{Michalewicz $(d=5, q=16)$, win-rate @ $t = 30$s}\label{tab:mic_wr}
\begin{tabular}{lcccccccc}
\hline
 & LogEI & TS & LP-UCB & LLP-UCB & AEGIS & KB-LogEI & KB-UCB & E-LogEI \\
\hline
UCB   & 0.75 & 0.85 & 0.75 & 1.0 & 1.0 & 0.85 & 0.55 & 0.8 \\
LogEI &  --  & 0.6  & 0.45 & 0.8 & 0.6 & 0.5  & 0.2  & 0.4 \\
\hline
\end{tabular}
\end{table}

\begin{table}[H]
\centering
\caption{Ackley $(d=10, q=8)$, MWU @ $t=30$s}\label{tab:acc_mwu}
\begin{tabular}{lcccccccc}
\hline
 & LogEI & TS & LLP-UCB & AEGIS & KB-LogEI & KB-UCB & E-LogEI \\
\hline
UCB
& $1.79 \times 10^{-4}$
& $1.06 \times 10^{-7}$
& $9.17 \times 10^{-8}$
& $9.13 \times 10^{-7}$
& $9.75 \times 10^{-6}$
& $2.07 \times 10^{-2}$
& $6.67 \times 10^{-6}$ \\

LogEI
& --
& $1.44 \times 10^{-4}$
& $5.09 \times 10^{-4}$
& $5.98 \times 10^{-1}$
& $9.68 \times 10^{-1}$
& $3.15 \times 10^{-2}$
& $9.03 \times 10^{-1}$ \\
\hline
\end{tabular}
\end{table}

\begin{table}[H]
\centering
\caption{Michalewicz $(d=5, q=16)$, MWU @ $t=30$s}\label{tab:mic_mwu}
\begin{tabular}{lcccccccc}
\hline
 & LogEI & TS  & LLP-UCB & AEGIS & KB-LogEI & KB-UCB & E-LogEI \\
\hline
UCB
& $1.12 \times 10^{-3}$
& $8.29 \times 10^{-5}$
& $4.54 \times 10^{-7}$
& $5.87 \times 10^{-6}$
& $2.60 \times 10^{-5}$
& $3.65 \times 10^{-1}$
& $2.75 \times 10^{-4}$ \\

LogEI
& --
& $5.61 \times 10^{-1}$
& $4.60 \times 10^{-4}$
& $3.37 \times 10^{-1}$
& $7.15 \times 10^{-1}$
& $1.14 \times 10^{-2}$
& $5.43 \times 10^{-1}$ \\
\hline
\end{tabular}
\end{table}

\section{Compute}
\label[appendix]{sec:app_comp}
All experiments were conducted on a single NVIDIA V100 GPU with 32 GB of memory. Each synthetic experiment required approximately 8 hours, while the real-world tasks ranged from about 8 hours (Pest Control) to 48 hours (CNN). Naturally, this was preceded by many hours of testing and prototyping.

\end{document}